\documentclass[letterpaper, 10 pt, conference]{ieeeconf}
\pdfoutput=1
\IEEEoverridecommandlockouts
\makeatletter
\let\NAT@parse\undefined
\makeatother

\usepackage[usenames,dvipsnames]{xcolor} 
\usepackage[numbers,sort&compress]{natbib} 
\usepackage{tikz}
\usepackage{float}
\usepackage[
hidelinks,
breaklinks,
colorlinks,
linkcolor=black,
citecolor=black,
filecolor=red,
urlcolor=blue
]{hyperref}
\usepackage{subcaption}
\usepackage{soul} 
\usepackage{balance}
\usepackage{comment}
\usepackage{tabularx}
\usepackage{booktabs}
\usepackage{pgfplots}
\usepackage{pgfplotstable}
\usetikzlibrary{arrows.meta}
\usetikzlibrary{backgrounds}
\usepackage{xparse} 

\usepackage{caption}
\usepackage[pagewise]{lineno}
\usepackage{amsmath,mathrsfs,amssymb}
\DeclareMathOperator*{\argmax}{arg\,max}

\usepackage{graphicx}
\graphicspath{{./figures/}}
\usepackage[numbers,sort&compress]{natbib}

\pgfplotsset{compat=newest}
\pgfplotsset{plot coordinates/math parser=false}

\usepackage{amsthm}

\usepackage{thmtools,thm-restate}

\newlength\figureheight
\newlength\subfiguresheight
\newlength\figurewidth

\newtheorem{theorem}{Theorem}

\newtheorem{corollary}{Corollary}[theorem]
\newtheorem{remark}{Remark}
\newtheorem{definition}{Definition}

\newif\ifextended
\extendedtrue

\newif\ifshowdevelopment
\showdevelopmentfalse

\newcommand{\ifdev}[1]{%
  \ifshowdevelopment%
    #1%
  \else\fi%
}

\newcommand{\micah}[1]{\ifdev{\textcolor{purple}{#1}}}

\newcommand{\skyler}[1]{\ifdev{\textcolor{Fuchsia}{#1}}}

\newcommand{\pixelconstant}{\alpha}

\newcommand{\spacefont}[1]{{\mathbb{#1}}}
\newcommand{\real}{{\spacefont{R}}}

\newcommand{\setsystemfont}[1]{{\mathscr{#1}}}
\newcommand{\ground}{\Omega}
\newcommand{\independence}{\setsystemfont{I}}
\newcommand{\block}{\setsystemfont{U}}

\newcommand{\setfun}{g}

\newcommand{\realfun}{\psi}

\newcommand{\objective}{\setfun^\mathrm{obj}}

\newcommand{\viewreward}{\setfun^\mathrm{view}}
\newcommand{\pathreward}{\setfun^\mathrm{path}}
\newcommand{\pixelterm}{\setfun^\mathrm{pixels}}

\newcommand{\pathweight}{w_\mathrm{p}}

\newcommand{\opt}{{\mathrm{opt}}}
\newcommand{\greedy}{{\mathrm{g}}}
\newcommand{\multiround}{{\mathrm{mr}}}

\newcommand{\numrobots}{{N^\mathrm{r}}}
\newcommand{\robotstates}{{\mathcal{X}^\mathrm{r}}}
\newcommand{\numactors}{{N^\mathrm{a}}}
\newcommand{\numrounds}{{N^\multiround}}
\newcommand{\robots}{{\mathcal{R}}}
\newcommand{\actors}{{\mathcal{A}}}
\newcommand{\faces}{{\mathcal{F}}}
\newcommand{\timespace}{{\mathcal{T}}}

\newcommand{\yaw}{{\theta}}
\newcommand{\declination}{{\phi}}
\newcommand{\viewspan}{{\gamma}}

\title{\huge
Multi-Robot Planning for Filming Groups of Moving Actors Leveraging Submodularity and Pixel Density
}

\author{Skyler Hughes$^{1}$, Rebecca Martin$^{2}$, Micah Corah$^{3}$, and Sebastian Scherer$^{2}$
\thanks{$^{1}$Skyler Hughes is an independent researcher.
  Skyler was sponsored by the NSF REU program as a part of the
  Robotics Institute Summer Scholars (RISS) program at Carnegie Mellon
  University.
  {\tt\small skyler\_hughes@outlook.com}
}
\thanks{$^{2}$Rebecca Martin and Sebastian Scherer are affiliated
  with the Airlab in the Robotics Institute at Carnegie Mellon University.
  {\tt\small \{rebecca2, basti\}@andrew.cmu.edu}
}%
\thanks{$^{3}$Micah Corah is affiliated
  with the Department of Computer Science at the Colorado School of Mines.
  Micah primarily contributed to this work while a postdoc at CMU.
  {\tt\small micah.corah@mines.edu}
}%
\thanks{
  This work is supported by the National Science Foundation (NSF) under Grant
  No. 2024173.
}
}

\tikzset{every picture/.style={line width=0.75pt}} 
\begin{document}

\maketitle
\thispagestyle{empty}
\pagestyle{empty}

\begin{abstract}
Observing and filming a group of moving actors with a team of aerial robots
is a challenging problem that combines elements of
multi-robot coordination,
coverage,
and view planning.
A single camera may observe multiple actors at once, and a robot team may
observe individual actors from multiple views.
As actors move about, groups may split, merge, and reform, and robots filming
these actors should be able to adapt smoothly to such changes in actor formations.
Rather than adopt an approach based on explicit formations or assignments, we propose an
approach based on optimizing views directly.
We model actors as moving polyhedra and compute approximate pixel densities for
each face and camera view.
Then, we propose an objective that exhibits diminishing
returns as pixel densities increase from repeated observation.
This gives rise to a multi-robot perception planning problem that we solve via
a combination of value iteration and greedy submodular maximization.
We evaluate our approach on challenging scenarios modeled after various
social behaviors and featuring different numbers of robots and actors and observe
that robot assignments and formations arise implicitly given the movements of
groups of actors.
Simulation results demonstrate that our approach consistently outperforms
baselines,
and in addition to performing well
with the planner's approximation of pixel densities our approach also performs
comparably for evaluation based on rendered views.
Overall, the \emph{multi-round} variant of the sequential planner we propose
meets (within 1\%) or exceeds \emph{formation} and \emph{assignment} baselines in all
scenarios.
\end{abstract}

\begin{figure}[t]
    \centering
    \tikzset{every picture/.style={line width=0.5pt}} 
    \input{figures/filming_cartoon.tex}
    \caption{A team of robots work together to film an unstructured scene with multiple moving actors splitting and merging.}
    \label{fig:my_label}
\end{figure}

\section{Introduction}



Uncrewed aerial vehicles (UAVs) are widely applicable as mobile sensing and
camera platforms.
UAVs provide the ability to position a camera anywhere in 3D space, opening up
the door to many possibilities across cinematography, inspection, and search and
rescue.
Because of this, UAVs are uniquely suited to capture complex scenarios such
as team sports or animal behaviors.
Additionally, when operating in a team, UAVs are able to capture multiple
viewpoints simultaneously which can enable
filming~\citep{Bucker_Bonatti_Scherer_2021}
or
reconstructing~\citep{Ho_Jong_Freeman_Rao_Bonatti_Scherer_2021,saini2019markerless,aircaprl}
one or more people from multiple perspectives.
Filming an unstructured group of actors introduces additional
complexity---groups may split, spread out, and reform in various ways.
This complexity can be ameliorated when this group motion is structured.
Groups that move as a unit could be abstracted as a single actor and filmed by a
team of UAVs flying in
formation~\citep{Ho_Jong_Freeman_Rao_Bonatti_Scherer_2021,saini2019markerless}.
Alternatively, an assignment scheme~\citep{buckman2019partial} could match UAVs
to individual actors or groups that move together.
Yet, groups of people in relevant settings---such as a team sport, a dance,
a race---may frequently split, rejoin, and reorganize in ways that are prone to
break formation and assignment schemes.
This motivates development of systems that can optimize the robots' collective
views more directly.

\subsection{Related work}

A large part of this work revolves around design of an objective for filming
groups of moving actors with multiple robots.
Mapping or exploring unknown environments is often closely tied to
reconstruction
(e.g. in the form of a mesh).
These works differ from our approach in that they typically represent the
environment with an occupancy grid, but objectives that capture observations of
surfaces or that include terms based on distance or incidence angle have similar
intent as what we propose~\citep{yoder2016fsr,bircher2018receding}.
Papers that focus on reconstruction of known static or dynamic
environments~\cite{maboudi2023review,song2021tro,Roberts_Shah_Dey_Truong_Sinha_Kapoor_Hanrahan_Joshi_2017,jiang_onboard_2023}
are also closely related to our approach.
From the perspective of perceptual reward,
\citet{Roberts_Shah_Dey_Truong_Sinha_Kapoor_Hanrahan_Joshi_2017}
propose a submodular objective based on covering hemispheres of viewing angles
for each face with increasing hemispherical coverage given to closer views.
An objective such as this could be applied to our
setting by summing over time-steps and treating each instant like a static
scene.
The approach we propose does not directly reward different viewing angles of
individual surface elements but
instead seeks to maximize the collective density of pixels,
following~\citep{jiang_onboard_2023}.
Additionally, like
\citet{Roberts_Shah_Dey_Truong_Sinha_Kapoor_Hanrahan_Joshi_2017}, we propose a
formulation based on a submodular objective.
Our Square-Root-PPA objective is a submodular variation of
the Pixel-Per-Area (PPA) objective by~\citet{jiang_onboard_2023}.

Given the design of the perception objective, the robots must plan to
maximize that objective.
Submodular optimization~\citep{nemhauser1978,fisher1978}
can solve many multi-robot perception, sensing and coverage problems
with guarantees on
suboptimality~\citep{hollinger2009ijrr,Singh_Krause_Guestrin_Kaiser_2009,atanasov2015icra,schlotfeldt2021tro}.
Often, this takes the form of a guarantee that the worst case perception
reward will be no worse than half of optimal~\citep{fisher1978}.
Then, like \citet{Bucker_Bonatti_Scherer_2021} we solve the single-robot
sub-problems optimally with value iteration on a directed graph which is
possible because view rewards for individual robots form a sum.
Because rewards for individual robots are additive and not submodular, methods for
single-robot informative (submodular) path planning are not relevant to this
single robot
subproblem~\citep{zhang2016aaai,Singh_Krause_Guestrin_Kaiser_2009,chekuri2005},
and submodularity is only relevant to the multi-robot aspect.
Alternatively, other methods for solving multi-robot active perception problems
such as Dec-MCTS~\citep{best2019ijrr} could be applicable.
However, solving the single-robot perception planning problem by methods like
Monte-Carlo tree
search~\citep{best2019ijrr,Corah_Michael_2019,corah2021iros,lauri2015ras} would
not be necessary for the same reason as before.

Finally, the aforementioned methods optimize routes and views directly.
We find that such methods provide flexibility in movement and capacity to
cover different actors at different times.
However, existing methods for filming and reconstruction rely on
formations~\citep{Ho_Jong_Freeman_Rao_Bonatti_Scherer_2021,saini2019markerless}
or controllers centered on the subjects being
filmed~\citep{xu2022iros,Bucker_Bonatti_Scherer_2021}.
Naively extending these approaches to the multi-robot setting
may not behave well if the people being filmed do not move as like one person.

\subsection{Contributions}

This work develops methods for perception planning with application to
videography that coordinate multiple UAVs to obtain diverse views of multiple
moving actors.
We present an objective that approximates pixel densities over the surfaces of
the actors that exhibits diminishing returns with repeat observation.
This extends the pixel-per-area (PPA) density approximation
that
\citet{jiang_onboard_2023}
propose with a mechanism for obtaining diminishing returns
(producing Square-Root-PPA or SRPPA).
This, in turn, enables multi-robot application.
Likewise, our choice of the SRPPA objective extends the planning approach by
\citet{Bucker_Bonatti_Scherer_2021} to the multi-actor case by enabling
reasoning about the quality of views of different actors.
Moreover, our analysis proves that SRPPA is monotonic and submodular.
This enables application of submodular maximization methods to
optimize views across the multi-robot team.
The results compare submodular maximization methods
to formation and assignment baselines in a variety of scenarios that simulate
challenging behaviors (splitting, merging, reorganizing) and social
scenarios (public speaking, a race).
Our approach meets (within 1\%) or exceeds the performance of our baselines for
all scenarios
and behaves intuitively such as by implicitly producing formations or
assignments.
The implementation and code to run experiments are also available
online.\footnote{%
  \url{https://github.com/castacks/MultiDroneMultiActorFilming}
}

Additionally, an extension of this work introduces non-collision
constraints between robots and an objective implementation based on rendered
views~\citep{suresh2024iros}.
While the key contribution by~\citet{suresh2024iros} is application to a more
realistic setting,
this paper's unique contribution includes general analysis of submodularity of
SRPPA and similar objectives---this is non-trivial because SRPPA does not
readily reduce to a form of coverage---and evaluation in an unconstrained
setting where suboptimality guarantees hold strictly and with scenarios that
emphasize the role of cooperative view planning.

\micah{Add citations to new Schwager papers}

\micah{Cite the Guiseppe Loianno paper that develops a similar reprojection
objective.}


\section{Background}
\label{section:background}

In this work, we apply methods for submodular optimization to coordinate
multi-robot teams and for analysis of our perception objective.
To begin, consider a set $\ground = \bigcup_{r\in 1:N} \block_r$ that might
represent possible assignments of actions to robots where each $\block_r$
represents a (disjoint) local set of actions associated with each robot
$r \in \robots = \{1,\ldots,\numrobots\}$.
The following subsections
build tools that operate on these sets when solving planning problems.

\subsection{Submodularity and monotonicity}
\label{section:background_submodularity}

The objectives in the perception planning problems we study
are \emph{set functions}, functions that
map sets of actions to real numbers
$\setfun : 2^\ground \rightarrow \real$, and we interpret the value of a set
function as a reward.
We will be interested in maximizing set functions subject to certain
conditions.
Generally, we consider set functions that are
\emph{normalized}
$\setfun(\emptyset)=0$,
\emph{monotonic}
$\setfun(A) \geq \setfun(B)$ for $B \subseteq A \subseteq \ground$,
and \emph{submodular}
$\setfun(A \cup c) - \setfun(A) \leq \setfun(B \cup c) - \setfun(B)$
where $c \in \ground \setminus A$.
Monotonicity expresses the notion that more observations produce more
reward or equivalently that the discrete derivative is positive.
Submodularity is a monotonicity condition on the second discrete
derivative~\citep{foldes2005} and expresses the notion that marginal gains
decrease given more prior observations.
Additionally, we will take some liberties with notation such as to replace
unions with commas or to implicitly wrap elements in sets.
For example, we abbreviate marginal gains as follows
$\setfun(c|A) = \setfun(A,c) - \setfun(A) = \setfun(A \cup \{c\}) - \setfun(A)$
which we read as ``the marginal gain for $c$ given $A$.''

\subsection{Submodular optimization for multi-robot coordination}
\label{sec:background_submodular_maximization}

In the problems we study, valid plans for the multi-robot team consist of
at most one action from each robot's local set.
That is
$X \in \independence =
\{X \subseteq \ground \mid 1 \geq |X \cap \block_r|, \forall r\in \robots\}$
which forms a
\emph{simple partition matroid}~\citep[Sec.~39.4]{schrijver2003}.
Thus, we wish to solve optimization problems of the form
\begin{align}
  X^{\opt} \in \argmax_{x \in \independence}
  \objective(X)
  \label{eq:submodular_maximization}
\end{align}
where $\objective$ is normalized, monotonic, and submodular.
Greedy algorithms can solve these problems with various guarantees of solutions
within a fraction of optimal, and these methods have been applied frequently to
solve planning problems related to perception, coverage, and
search~\citep{Singh_Krause_Guestrin_Kaiser_2009,atanasov2015icra,hollinger2009ijrr}


\section{Problem formulation}

Consider a team of $\numrobots=|\robots|$ robots  with states
$x_{r,t} \in \robotstates$
for $r\in\robots$, where $\robotstates$ is a subset of the special euclidean group
$SE(2)$, and a set of actions $u_{r,t} \in U_{r,t}$, where $U_{r,t}$ is the
\textit{finite} space of actions available to robot $r$ at time $t \in
\{1,\ldots,T\} = \timespace$.
The robots film the actors $\actors=\{1,\ldots,\numactors\}$
which have states
$y_{a,t} \in SE(3)$, and we refer to an actor's trajectory as
$Y_a = [y_{a,1},\ldots, y_{a,T}]$.
\textbf{Problem:}
Given known (or scripted) actor motions, the task is to select robot
trajectories to maximize the objective $\objective$
which primarily represents the quality of the robots' views of the actors over
the duration of the time horizon.

\subsection{Motion model}
State transitions are governed by the \textit{motion model}:
\begin{equation}
    x_{r, t+1} = f(x_{r,t}, u_{r,t})
    \label{eq:transition_model}
\end{equation}
where $f$ constrains motion to state transitions within distance $D$ of the
current state, with one of eight possible camera orientations.
Additionally, we will refer to the robots' yaw angle as $\yaw_{r,t}$.
Robots are able to rotate clock-wise or counter-clock wise by a $\pi/4$ radians
at each time-step.
The space of control actions $U_{r,t}$ thus encodes a finite list of adjacent
positions and orientations, and we will treat this set as time-varying only to
account for in-valid transitions---in our case we require the robot to remain
in-bounds on a grid.

\subsection{Actor motion and representation}
\label{}

The robots seek to collectively film the actors which follow
\emph{known trajectories} $Y_a$.
The actors (such as people, animals, or cars)
are represented by simple polyhedra, in our case, a capped hexagonal prism.
Each actor $a$ consists of a set of faces $F_a \subseteq \faces$ where
$\faces$ is the total set of faces in the scene.
Faces are parameterized by the normal $\vec{\eta_f}$, the area $A_f$, the
face center position $\vec{p_f}$, and the weight $w_f$.
We define the position and orientation of the faces in $F_a$ relative to the
actor
position $\vec{c_{a,t}}$ and rotation $\theta_{a,t} \in [0, 2\pi]$
(and the corresponding rotation matrix $R_{a,t}$).
\textit{See  Fig.~\ref{fig:scene}}.
Actor trajectories are not constrained in position or orientation.


\begin{figure}
    \centering
    \tikzset{every picture/.style={line width=0.75pt}} 
\resizebox{250.0pt}{!}{%
\begin{tikzpicture}[x=0.75pt,y=0.75pt,yscale=-1,xscale=1]

\draw    (136.2,178.97) -- (151.93,202.7) ;
\draw    (151.93,202.7) -- (200.33,207.5) ;
\draw    (200.33,207.5) -- (232.33,189.1) ;
\draw    (232.33,189.1) -- (217.13,165.1) ;
\draw    (217.13,165.1) -- (169.93,160.3) ;
\draw    (136.2,178.97) -- (169.93,160.3) ;
\draw    (136.2,178.97) -- (136.73,272.5) ;
\draw    (151.93,202.7) -- (151.13,294.5) ;
\draw    (136.73,272.5) -- (151.13,294.5) ;
\draw    (200.33,207.5) -- (198.87,299.63) ;
\draw    (151.13,294.5) -- (198.87,299.63) ;
\draw    (232.33,189.1) -- (232.2,283.3) ;
\draw    (198.87,299.63) -- (232.2,283.3) ;
\draw [line width=1.5]    (217.27,244.9) -- (257.77,258.92) ;
\draw [shift={(260.6,259.9)}, rotate = 199.09] [color={rgb, 255:red, 0; green, 0; blue, 0 }  ][line width=1.5]    (17.05,-5.13) .. controls (10.84,-2.18) and (5.16,-0.47) .. (0,0) .. controls (5.16,0.47) and (10.84,2.18) .. (17.05,5.13)   ;
\draw  [fill={rgb, 255:red, 0; green, 0; blue, 0 }  ,fill opacity=1 ] (179.87,274.3) .. controls (179.87,272.46) and (181.36,270.97) .. (183.2,270.97) .. controls (185.04,270.97) and (186.53,272.46) .. (186.53,274.3) .. controls (186.53,276.14) and (185.04,277.63) .. (183.2,277.63) .. controls (181.36,277.63) and (179.87,276.14) .. (179.87,274.3) -- cycle ;
\draw  [dash pattern={on 0.84pt off 2.51pt}]  (232.33,281.77) -- (217.13,257.77) ;
\draw  [dash pattern={on 0.84pt off 2.51pt}]  (217.13,257.77) -- (169.93,252.97) ;
\draw  [dash pattern={on 0.84pt off 2.51pt}]  (136.2,271.63) -- (169.93,252.97) ;
\draw  [dash pattern={on 0.84pt off 2.51pt}]  (170.27,161.03) -- (169.47,252.83) ;
\draw  [dash pattern={on 0.84pt off 2.51pt}]  (218.6,165.63) -- (217.13,257.77) ;
\draw    (183.2,274.3) -- (215.75,246.21) ;
\draw [shift={(217.27,244.9)}, rotate = 139.21] [color={rgb, 255:red, 0; green, 0; blue, 0 }  ][line width=0.75]    (10.93,-3.29) .. controls (6.95,-1.4) and (3.31,-0.3) .. (0,0) .. controls (3.31,0.3) and (6.95,1.4) .. (10.93,3.29)   ;
\draw    (431.57,86.33) -- (455.19,93.12) ;
\draw    (455.19,93.12) -- (481.42,88.09) ;
\draw    (481.42,88.09) -- (484.09,76.56) ;
\draw    (484.09,76.56) -- (460.55,69.59) ;
\draw    (460.55,69.59) -- (434.89,74.45) ;
\draw    (431.57,86.33) -- (434.89,74.45) ;
\draw    (457.28,80.21) -- (427.85,100.66) ;
\draw    (457.28,80.21) -- (517.53,85.95) ;
\draw    (457.28,80.21) -- (400.97,76.4) ;
\draw    (457.28,80.21) -- (488.93,61.82) ;
\draw   (428.25,112.58) .. controls (412.17,111.94) and (398.96,106.08) .. (398.74,99.5) .. controls (398.52,92.91) and (411.38,88.1) .. (427.45,88.74) .. controls (443.53,89.38) and (456.74,95.24) .. (456.96,101.82) .. controls (457.18,108.41) and (444.33,113.22) .. (428.25,112.58) -- cycle ;
\draw   (517.93,97.87) .. controls (501.85,97.23) and (488.64,91.37) .. (488.42,84.78) .. controls (488.2,78.2) and (501.05,73.38) .. (517.13,74.03) .. controls (533.21,74.67) and (546.42,80.53) .. (546.64,87.11) .. controls (546.86,93.7) and (534.01,98.51) .. (517.93,97.87) -- cycle ;
\draw   (489.33,73.74) .. controls (473.25,73.09) and (460.04,67.24) .. (459.82,60.65) .. controls (459.6,54.07) and (472.45,49.25) .. (488.53,49.89) .. controls (504.61,50.54) and (517.82,56.4) .. (518.04,62.98) .. controls (518.26,69.56) and (505.4,74.38) .. (489.33,73.74) -- cycle ;
\draw   (401.37,88.33) .. controls (385.29,87.68) and (372.08,81.82) .. (371.86,75.24) .. controls (371.64,68.66) and (384.49,63.84) .. (400.57,64.48) .. controls (416.65,65.13) and (429.86,70.98) .. (430.08,77.57) .. controls (430.3,84.15) and (417.44,88.97) .. (401.37,88.33) -- cycle ;

\draw  [dash pattern={on 0.84pt off 2.51pt}]  (457.55,80.69) -- (455,277.67) ;
\draw  [dash pattern={on 3.75pt off 3pt on 7.5pt off 1.5pt}]  (457.55,80.69) -- (333.4,191) ;
\draw [color={rgb, 255:red, 0; green, 0; blue, 0 }  ,draw opacity=1 ] [dash pattern={on 3.75pt off 3pt on 7.5pt off 1.5pt}]  (455.78,271.71) -- (233,160.63) ;
\draw  [draw opacity=0] (367.98,295.61) .. controls (359.36,287.65) and (357.48,277.36) .. (364.47,267.4) .. controls (371.88,256.82) and (387.58,249.04) .. (405.6,245.77) -- (419.63,276.96) -- cycle ; \draw   (367.98,295.61) .. controls (359.36,287.65) and (357.48,277.36) .. (364.47,267.4) .. controls (371.88,256.82) and (387.58,249.04) .. (405.6,245.77) ;  
\draw  [dash pattern={on 4.5pt off 4.5pt}]  (183.2,274.7) .. controls (487.75,234.25) and (110.2,359.4) .. (434.6,358.2) ;
\draw [shift={(434.6,358.2)}, rotate = 179.79] [color={rgb, 255:red, 0; green, 0; blue, 0 }  ][line width=0.75]    (10.93,-3.29) .. controls (6.95,-1.4) and (3.31,-0.3) .. (0,0) .. controls (3.31,0.3) and (6.95,1.4) .. (10.93,3.29)   ;
\draw [line width=0.75]  [dash pattern={on 4.5pt off 4.5pt}]  (58.6,212.2) .. controls (182.6,189.8) and (13,297.8) .. (183.2,274.7) ;
\draw [color={rgb, 255:red, 0; green, 0; blue, 0 }  ,draw opacity=1 ] [dash pattern={on 3.75pt off 3pt on 7.5pt off 1.5pt}]  (455.78,271.71) -- (135.5,363.63) ;
\draw  [draw opacity=0] (355.93,169.91) .. controls (346.29,158.6) and (338.41,145.45) .. (332.87,130.74) .. controls (326.71,114.36) and (324,97.52) .. (324.4,81.03) -- (456.55,84.19) -- cycle ; \draw   (355.93,169.91) .. controls (346.29,158.6) and (338.41,145.45) .. (332.87,130.74) .. controls (326.71,114.36) and (324,97.52) .. (324.4,81.03) ;  
\draw [color={rgb, 255:red, 0; green, 0; blue, 0 }  ,draw opacity=1 ] [dash pattern={on 3.75pt off 3pt on 7.5pt off 1.5pt}]  (457.55,80.69) -- (305,82.13) ;

\draw (462,167.4) node [anchor=north west][inner sep=0.75pt]  [font=\huge]  {$d_{h}$};
\draw (339.5,105.12) node [anchor=north west][inner sep=0.75pt]  [font=\Huge]  {$\phi $};
\draw (435.93,21.2) node [anchor=north west][inner sep=0.75pt]  [font=\huge]  {$p_{r}$};
\draw (150.67,268.83) node [anchor=north west][inner sep=0.75pt]  [font=\Huge]  {$c_{y}$};
\draw (242.47,215.57) node [anchor=north west][inner sep=0.75pt]  [font=\huge]  {$\eta _{f}$};
\draw (364,322.07) node [anchor=north west][inner sep=0.75pt]  [font=\huge]  {$Y_{a}$};
\draw (391,255.8) node [anchor=north west][inner sep=0.75pt]  [font=\Huge]  {$\gamma $};
\draw (166,224.49) node [anchor=north west][inner sep=0.75pt]  [font=\Huge]  {$p_{f}$};

\end{tikzpicture}

}
    \caption{Scene representation: Each actor $a\in\actors$
      is modeled as a hexagonal prism and follows a trajectory $Y_a$ in
      the plane.
      Robots move on a grid at height $d_h$ and carry a camera with
      field-of-view $\viewspan$ and declination $\declination$.
      \micah{Update this caption.}
    }
    \label{fig:scene}
\end{figure}

\subsection{Camera and sensor model}

The robots are equipped with cameras with which they observe and film the
actors.
For the purpose of planning we adopt a simplified quasi-two-dimensional camera model.
In this model,
cameras face forward and observe over a horizontal viewing angle $\viewspan$ (field-of-view) regardless of distance.
While we do not model camera view frustums or occlusions,
we do account for $d_h$, the height of the robots above the actors, in distance calculations between robots and actor faces.
Additionally, we cull back-sides of faces by
checking whether the face normal is facing away from the camera.
Otherwise, we do not do raytracing or account for occlusions between actors.
Later, we will model view quality based on an approximation of the size of the
face projected into the image sensor as determined by the distance and a check
regarding whether a given face is \emph{in view}.
We also compare this simple camera model against a more realistic one in the
rendering based evaluations in Sec.~\ref{sec:rendering_evaluation}.



\subsection{Assumptions}

For the purpose of this paper, we assume \emph{centralized computation}---this
is reasonable for many potential settings for videography systems
such as sports arenas that feature controlled environments and can be
instrumented with suitable communications equipment.
Alternatively, there are also applicable distributed solvers that are amenable
to constraints on network connectivity~\citep{corah2018cdc,xu2022cdc}.
Additionally, \emph{Robot and actor positions and orientations are known}, and
we require some additional instrumentation (e.g.
GPS~\citep{Ho_Jong_Freeman_Rao_Bonatti_Scherer_2021,bonatti2020jfr,saini2019markerless})
to track position.
Additionally, \emph{actor motions are known or scripted}.
Although there are scripted scenarios that are relevant to this work,
predictions for unscripted scenarios
(e.g. filming team sports)
may only be accurate over a short horizon
(e.g. based on a velocity output from a Kalman filter~\citep{bonatti2019iros})
or require learned predictions.
Likewise, objective evaluation (Sec.~\ref{sec:objective}) in expectation would
not compromise our analysis but could significantly increase computation time.
Finally, we \emph{ignore collisions and occlusions}; this work focuses on the
design of the SRPPA objective and behavior of the submodular optimization
problem.
Collision and occlusion-aware planning is the focus of our succeeding work
\citep{suresh2024iros}.

\subsection{Objective function}
\label{sec:objective}

The design of the perception objective $\objective$ for the videography task
captures the following intuition:
\begin{itemize}
  \item \emph{Maximum actor size}: Select camera views to maximize size of
    actors in the field of view
  \item \emph{Maximum actor coverage}: Robots should collectively keep all actors in view at all times and with uniform coverage quality
  \item \emph{View diversity}: Prefer views that cover different sides of an actor versus many views of the same side
  \item \emph{Actor centering}: Prefer centering actors in field of view
  \item \emph{Operator preference}: Prioritize different actors (speaker vs. crowd)
    or parts of their surfaces (face vs. body)
\end{itemize}

Following the description of the submodular maximization problem in
Sec.~\ref{sec:background_submodular_maximization},
we define the robots' local action sets via tuples representing the assignment
of a sequence of valid control actions
$\block_r = \{(r,u_{r,1:T})\mid u_{r,t} \in U_{r,t}\}$ to robot $r\in\robots$.
Based on these sequences of control actions,
we define the objective as a sum of \emph{path} $\pathreward$
and
\emph{view}
$\viewreward_{a,t}$
\emph{rewards} over actors and time-steps
\begin{align}
  \objective(X) &=
  \pathreward(X) +
  \sum_{t = 1}^{T}
  \sum_{f \in F_a,\,a \in \actors}
  \viewreward_{f,t}(X)
  \label{eq:objective}
\end{align}
for $X\subseteq\ground$ and given the predicted actor trajectories $Y_a$.
We call this the Square-Root PPA (Pixels-Per-Area~\citep{jiang_onboard_2023})
or SRPPA objective for reasons that will soon be clear.
The path reward $\pathreward(X) = \sum_{x\in X} \pathweight(x)$
is a reward ($\pathweight(x)\geq 0$) on the robots' paths.\footnote{%
  In our implementation, we define $\pathreward$
  to provide a small reward for each time-step where a robots' position, orientation or both do not
  change.
}

The function $\viewreward$ represents the \textit{total view quality}
for a face at a given time.
In order to satisfy suboptimality guarantees afforded by sequential
greedy
planning~\citep{nemhauser1978,fisher1978}
this view reward must be \textit{submodular} and \textit{monotonic}.
The following expression
applies a square-root in a way which will satisfy these requirements
\begin{align}
  \viewreward_{f,t}(X) = w_{f}A_{f}\sqrt{\pixelterm_{f,t}(X)}
    \label{eq:view_quality}
\end{align}
where $\pixelterm$ accumulates the sensing quality for a particular face $f$ of
an actor across all robots
and $w_{f}\geq 0$
is a designer weight which can prioritize actors (e.g. a speaker)
or individual mesh faces (e.g. a person's front or face).

The function $\pixelterm$ approximates the cumulative pixel density for a given
face, and we will additionally
weight this based on camera alignment.
However, we must first define a few more terms.
Defining the position of the robot as
$\vec{p_r}=[x_r,y_r,d_h]^\mathrm{T}$, the relative position of a particular face $f$ on actor $a$ is:
\begin{align}
  \vec{d_{r,f,t}} = (\vec{c_{a,t}} + R_{a,t}\vec{p_f}) - \vec{p_r}.
  \label{eq:dist}
\end{align}
Then, define the rotated face normal as
$\vec{\eta} = R_{a,t}\vec{\eta_{f}}$.
The weighted pixel density is as follows:
\begin{align}
  \begin{split}
    &\pixelterm_{f,t}(X)
    \\
    &=
    \sum_{(r,u_{r,1:T})\in X}
    \pixelconstant
    \text{INVIEW}_{f}(x_{r,t})
    \frac{(\vec{d_{r,f,t}} \cdot \vec{\eta})}{||d_{r,f,t}||^3}
    \frac{(\vec{d_{r,f,t}} \cdot \vec{\eta_r})}{||d_{r,f,t}||}
  \end{split}
    \label{eq:pixelterm}
\end{align}
where $||\cdot||$ is the 2-norm, $\text{INVIEW}$ returns whether a given face
is in the field of view and facing the robot and zero otherwise;
$\vec{\eta_r}$ is a unit vector representing the robot heading;
and $\pixelconstant$ is the number of pixels per unit area at one
meter.\footnote{%
  Referring to \eqref{eq:view_quality}, $\pixelconstant$
  acts as a scaling factor and does not affect optimality.
}
Here, the first ratio corresponds to
computation of pixels-per-area,\footnote{%
  The dot product computes the alignment of the face with the camera (as in a
  computation of flux).
  The denominator normalizes the distance and accounts for projected area
  diminishing with the square of distance.
}
and the second forms a weight that encourages
robots to center faces in the camera view.

\begin{remark}[Intuition for perception objective]
  In \eqref{eq:pixelterm}, $\pixelterm$ approximates the cumulative density of
  pixels
  (weighted to encourage centering)
  for all robots observing each face.
  If we only maximized the sum of these terms, there would be no diminishing
  returns, and robots might maximize rewards by observing only a single actor or
  face.
  The square-root in \eqref{eq:view_quality} introduces diminishing returns
  because growth of the square root slows as $\pixelterm$ increases
  (analysis in Sec.~\ref{sec:analysis} states this formally).
  Generally, this encourages robots to cover actors uniformly at a moderate
  level versus individually at high levels.
\end{remark}



\section{Planning approach}


In this work, we seek to maximize $\objective$ \eqref{eq:objective} by optimizing
robot trajectories given a set of actor trajectories.
To simplify the problem we define two distinct planning subproblems:
a single-robot planning subproblem, and a coordination subproblem.
In the single-robot planning step, we seek an optimal trajectory for a single
robot given actor trajectories.
In the coordination step, we maximize overall sensing quality across all robots
by sequencing multiple single-robot planning steps.
Our planning approach is similar to that of
\citet{Bucker_Bonatti_Scherer_2021}.

\subsection{Single robot planning}
\label{sec:single-robot}

We apply backward value iteration~\citep[Sec.~2.3.1.1]{lavalle2006planning}
to solve the single robot planning problem optimally similarly as other works
involving perception planning~\citep{%
Ho_Jong_Freeman_Rao_Bonatti_Scherer_2021,Bucker_Bonatti_Scherer_2021,atanasov2015icra}.
Backward value iteration operates by taking a single pass
over $\robotstates \times \timespace$ going backward in time.
For every state time pair, backward value iteration iterates over control
actions and selects the one that maximizes the immediate reward plus the reward
to the end of the horizon from the next state
(which has already been visited and computed).
By doing so, this produces a plan for a single robot that maximizes the perception objective
$\objective$ directly or marginally given other robots prior decisions as in the
next section.

\subsection{Sequential planning and coordination}
\label{sec:sequential}

We coordinate robots via sequential greedy planning.
Through this process, robots each plan as described in
Sec.~\ref{sec:single-robot} and maximize $\objective$ conditional on the prior
robots' selections.
Thus, the robots produce
a greedy solution $X^\greedy = \{x_1^\greedy,\ldots,x_r^\greedy\}$
by solving
\begin{align}
  x^\greedy_r \in \argmax_{x\in\block_r} \objective(x|X_{1:r-1}^\greedy)
  \label{eq:sequential_greedy}
\end{align}
in sequence via value iteration where
$X_{1:r-1}^\greedy$ is the set of prior selections.
\citet{fisher1978} proved the following suboptimality guarantee:
if $\objective$ is monotonic, submodular, and normalized, then
$\objective(X^\greedy) \geq \frac{1}{2} \objective(X^\greedy)$
given that $X^\opt$ is the optimal solution to
\eqref{eq:submodular_maximization}.

\subsection{Multiple rounds of greedy planning}
\label{sec:multi-round-greedy}

Although a single pass of greedy planning guarantees solutions no worse than
half of optimal, we are often able to improve these solutions in practice.
We adopt a similar approach as \citet[Sec.~4.2]{mccammon2021jfr} do for a surveying task.
Specifically, robots replan by solving a slightly different single-robot
subproblem
\begin{align}
  x^\multiround_r \in \argmax_{x\in\block_r}
  \objective(x|\hat X \setminus \block_r)
  \label{eq:multi_round_greedy}
\end{align}
where $\hat X \subseteq \ground$ is the solution we wish to improve.
Here, we modify \eqref{eq:sequential_greedy} by removing any assignment to
$r\in\robots$ to allow that robot to replan.
When planning in multiple rounds robots solve \eqref{eq:multi_round_greedy} in
$\numrounds$ passes over $\robots$.
By this process, robots first produce a solution equivalent to $X^\greedy$
and in subsequent rounds may improve solution quality to produce
$X^\multiround$.
This process cannot produce a worse solution, but there is no
guarantee of an improved or optimal solution either.

\section{Analysis}
\label{sec:analysis}

This section introduces the analysis of the monotonicity properties of the
objective and applies that to guarantee bounded solution quality.

\begin{restatable}[Monotonicity properties of SRPPA]
  {theorem}{monotonicitiesofsrppa}
  The SRPPA objective $\objective$ from \eqref{eq:objective} is normalized,
  monotonic, and submodular.
  Moreover, SRPPA satisfies alternating monotonicity conditions and is
  $n$-increasing for \emph{odd} values of $n$ or else $n$-decreasing if
  \emph{even}.
  \label{theorem:submodular_objective}
\end{restatable}

\begin{corollary}[Bounded suboptimality]
  \label{corollary:bounded_suboptimality}
  Theorem~\ref{theorem:submodular_objective} ensures that $\objective$
  satisfies the requirements stated in Sec.~\ref{sec:sequential}
  for sequential (and multi-round) planning
  to produce solutions that satisfy
  $\objective(X^\greedy) \geq \frac{1}{2} \objective(X^\greedy)$.
  Additionally, $\objective$ is 3-increasing
  which is sufficient to guarantee bounded suboptimality for
  distributed optimization methods~\citep{corah2020phd,xu2022cdc} such as via
  the RSP algorithm (in expectation)~\citep[Theorem~9]{corah2020phd}.
\end{corollary}

We include a full proof of Theorem~\ref{theorem:submodular_objective}
in
\ifextended
Appendix~\ref{appendix:proofs}.
\else
the appendix of the extended version~\citep{hughes2024extended}.
\fi
The following lemma
(which we prove in
\ifextended
Appendix~\ref{sec:composing_modular_real}%
\else
\citep[Appendix~I-C]{hughes2024extended}%
\fi)
is key to our approach.
We use this lemma to prove that $\viewreward$ \eqref{eq:view_quality}
transforms $\pixelterm$ \eqref{eq:pixelterm} such that the resulting function is
monotonic and submodular.
The reader may refer to
\ifextended
Appendix~\ref{sec:monotonicity}
\else
\citep[Appendix~I-A]{hughes2024extended}
\fi
or \citep{foldes2005} for discussion of general monotonicity conditions.

\begin{restatable}[Monotonicity for composing a real and a modular function]
  {lemma}{lemmarealcomposition}
  \label{lemma:real_composition}
Consider a monotonic, modular set function $\setfun : \ground \rightarrow
\real_{\geq0}$
and a real function $\realfun : \real_{\geq0} \rightarrow \real$.
If $\realfun$ is $m$-increasing (or decreasing)
according to
\ifextended
Def.~\ref{definition:real_monotonic},
\else
\citep[Def.~3]{hughes2024extended},
\fi
then their composition $\hat\setfun(x) = \realfun(\setfun(x))$
is also $m$-increasing (or decreasing).
\end{restatable}

The rest of the proof of Theorem~\ref{theorem:submodular_objective} is in
\ifextended
Appendix~\ref{sec:theorem_proof}.
\else
\citep[Appendix~I-D]{hughes2024extended}.
\fi
The outline of this proof is as follows.
First, we apply Lemma~\ref{lemma:real_composition} to $\viewreward$ and
$\pixelterm$.
Then, we apply results by \citet{foldes2005} to prove that the sum of terms in
$\objective$ \eqref{eq:objective} preserves monotonicity conditions.
Finally, $\objective$ is normalized ($\objective(\emptyset) \!=\! 0$) because
all terms are normalized.


\begin{remark}[Transformations with other real functions]
  Following Lemma~\ref{lemma:real_composition},
  real functions other than the square root could be applied to implement
  diminishing returns.
  The functions $1-e^{-x}$ and $\log(x+1)$ are two other examples with
  alternating derivatives.
  Other functions such as variations of sigmoids $\frac{1}{1+e^{-x}}$
  can produce monotonic submodular objectives that do not satisfy higher order
  monotonicity properties.
\end{remark}

Although $\objective$ may not appear to belong to any
specific class of functions aside from satisfying these monotonicity
conditions, we can identify some additional structure.

\begin{remark}[Relationship between coverage and alternating derivatives]
  In addition to being monotonic and submodular, $\objective$ satisfies alternating
  monotonicity conditions because derivatives of the square root have
  alternating signs (Theorem~\ref{theorem:submodular_objective}).
  We have remarked before that this same alternating derivatives property
  applies to a form of weighted coverage
  objective~\citep[Theorem~9]{corah2020phd}, and others have made similar
  observations~\citep{salek2010you,wang2015}.
  One may suspect that alternating derivatives are
  sufficient for a set function to be equivalent to weighted coverage though we
  are not yet aware of a published proof of this statement.
\end{remark}

\section{Methods}
This section introduces baseline planners we compare against and their
implementations, evaluation scenarios, and approaches to evaluating perception
performance.



\subsection{Planner baselines}

\subsubsection{Myopic planner}

Myopic planning refers to planning without coordination with the
rest of the team.
Specifically, robots run the single-robot planner (Sec.~\ref{sec:single-robot})
for themselves only without sharing results like in the sequential planning
scheme.
In~\citep{corah_performance_2022}, we observe that robots seeking to maximize view
quality without coordination may converge to the same views to the detriment
of global view rewards.

\subsubsection{Formation planner}

Formation planning is
a simple and effective approach for filming
groups of moving
actors~\citep{Ho_Jong_Freeman_Rao_Bonatti_Scherer_2021,saini2019markerless}.
For our baseline, we implemented a formation planner
that places robots evenly on a circle centered on the centroid of the actor
positions with a radius based on the distance of the furthest actor from
the center plus a safety margin.
In essence, we assume the group of actors behaves like a single
``meta-actor.''
To improve performance, each robot also focuses its view on the nearest actor.
This approach would be effective in many real world scenarios but
struggles as groups move apart or become less circular.

\subsubsection{Assignment planner}

With the assignment planner, each robot is tasked with observing only a subset
of the actors.
The assignment planner attempts to evenly distribute actors to robots and
produces an assignment that is \textit{fixed} for the entire planning horizon.
If the number of actors is larger than the number of robots, the set of actors
will be distributed equally among the robots, but some actors will not be
assigned if the set is not evenly divisible.
If the number of robots is larger than the number of actors, each robot
will be assigned a single actor with some overlapping assignments.






\subsection{Scenarios}

We evaluate our planners and baselines on a set of hand-crafted scenarios
(Fig.~\ref{fig:scenario_summary})
that mimic challenging multi-robot filming situations.
These scenarios primarily specify the actor trajectories and the number of
robots.
Two scenarios were designed to act as canonical points of
comparison, such as stationary targets in a tight cluster
(Fig.~\ref{subfig:cluster}) or a group uniformly separating
(Fig.~\ref{subfig:spreadout_group}), and
we expect similar performance across all planners for the \emph{cluster}
scenario.
Some scenarios are designed to target specific weaknesses of our baselines.
In the \textit{cross-mix} scenario, 6 actors start in pairs, then cross and mix
together such that two of the three groups have swapped members; this can be challenging for planners that rely on fixed robot-actor pairings.
The \textit{track-runners} and \textit{priority-runners} scenarios mimic real
life situations relevant to sports cinematography where groups periodically
spread and join or where one character, i.e. the leader in a race, may be a more
important subject for filming than others.
The \textit{priority-speaker} scenario mimics a gathering in which a
group of moving actors is addressed by a stationary speaker.
In each ``priority'' scenario, one actor is given a higher weight than the
others.
A desirable outcome is to focus more attention on
prioritized actors while obtaining fewer or more distant views of the rest.
Additionally, all
scenarios with significant group splitting (\textit{four-split},
\textit{split-and-join}, and \textit{spreadout-group}) pose a challenge for
applying formation planning~\citep{Ho_Jong_Freeman_Rao_Bonatti_Scherer_2021} to
multi-robot settings.

\def\imgwidth{1.5in}
\begin{figure*}
\centering
\setlength{\tabcolsep}{1pt}
\setlength{\figureheight}{1.2in}
\begin{tabular}{cccc}
  \begin{subfigure}[b]{0.24\linewidth}
    \centering
      \includegraphics[height=\figureheight, trim={100 100 100 100}, clip]{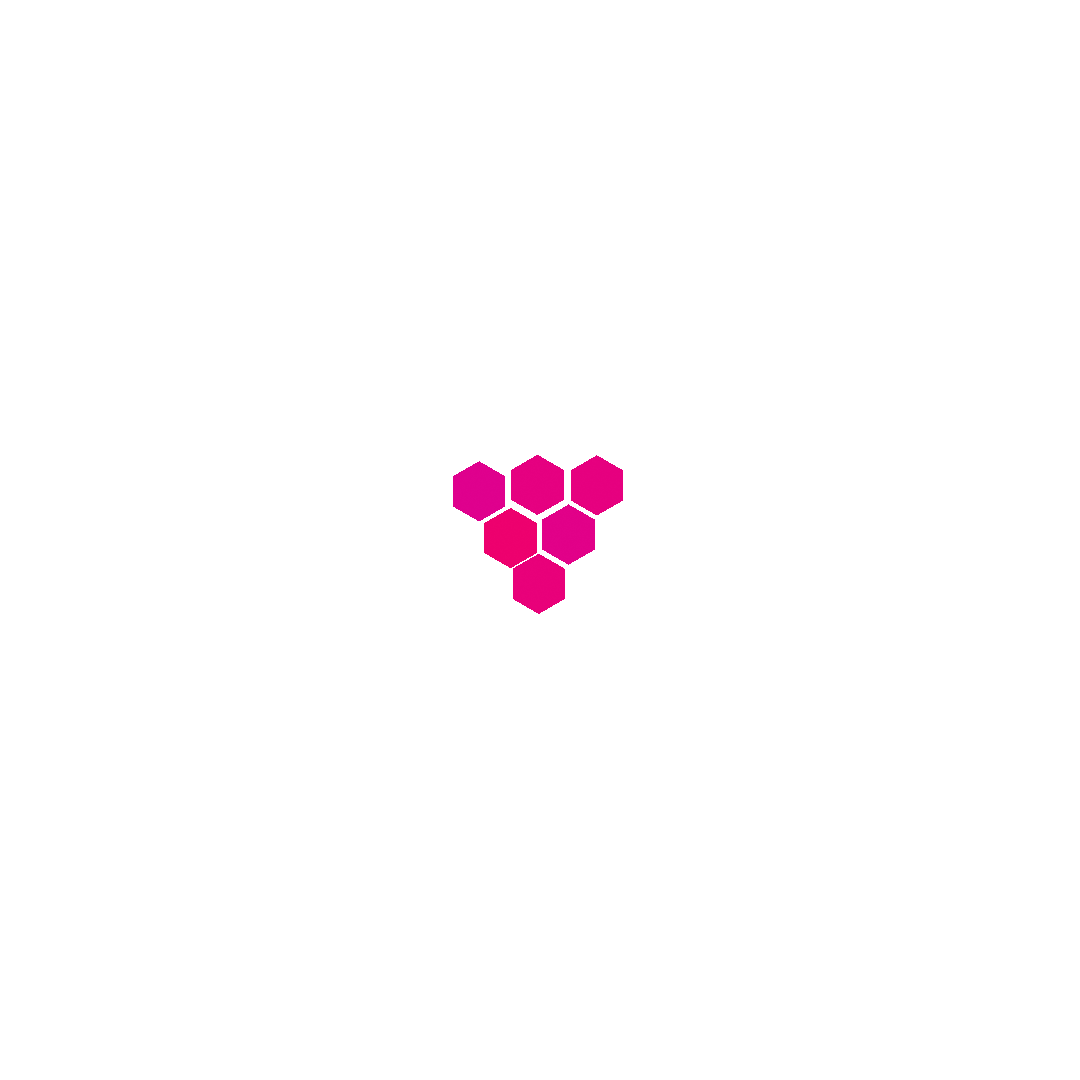}
    \caption{Cluster; $\numrobots\!=\!4$, $\numactors\!=\!6$}
    \label{subfig:cluster}
  \end{subfigure} &
  \begin{subfigure}[b]{0.24\linewidth}
    \centering
      \includegraphics[height=\figureheight, trim={50 150 50 50}, clip]{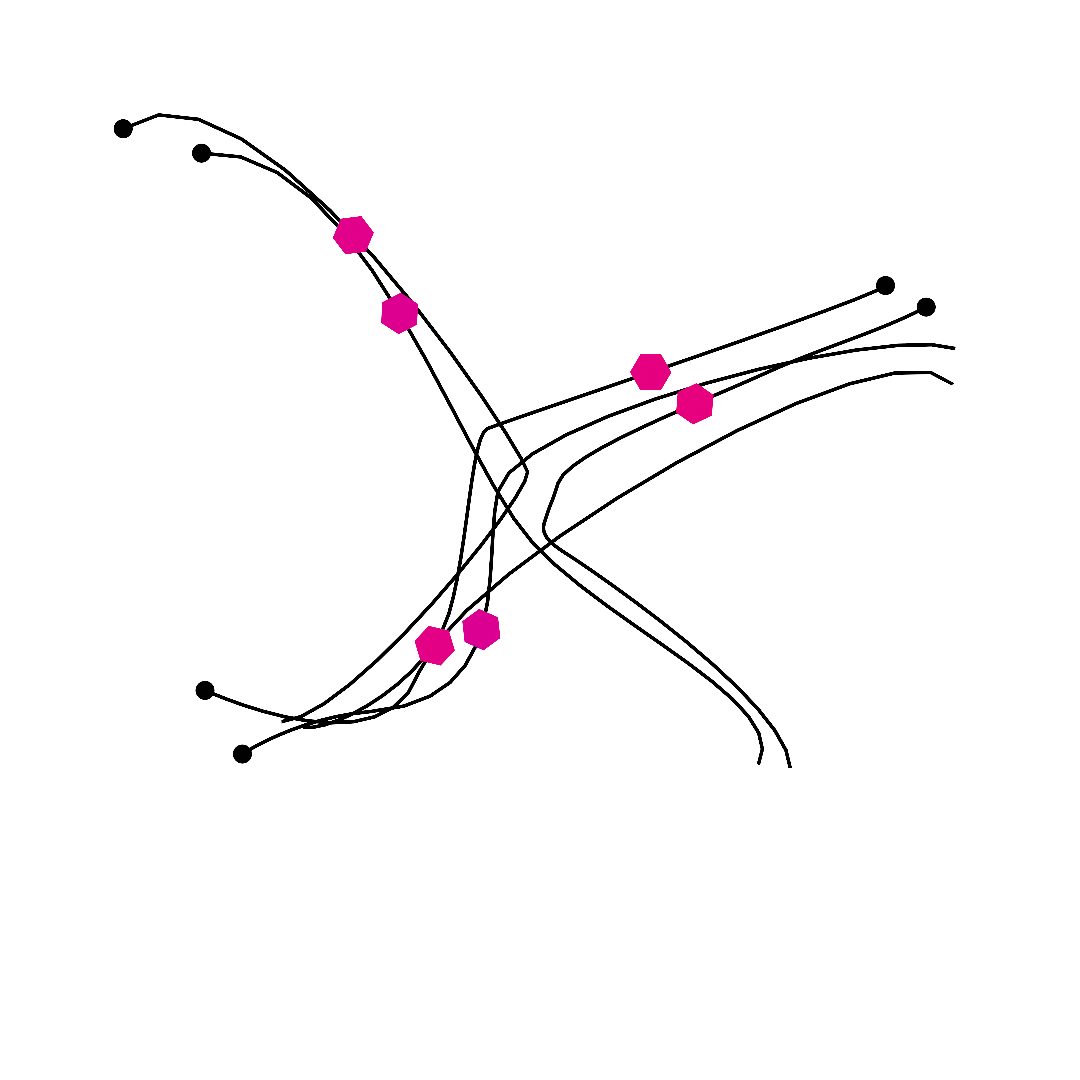}
    \caption{Cross-Mix; $\numrobots\!=\!4$, $\numactors\!=\!6$}
    \label{subfig:cross_mix}
  \end{subfigure} &
  \begin{subfigure}[b]{0.24\linewidth}
    \centering
      \includegraphics[height=\figureheight, trim={100 140 100 80}, clip]{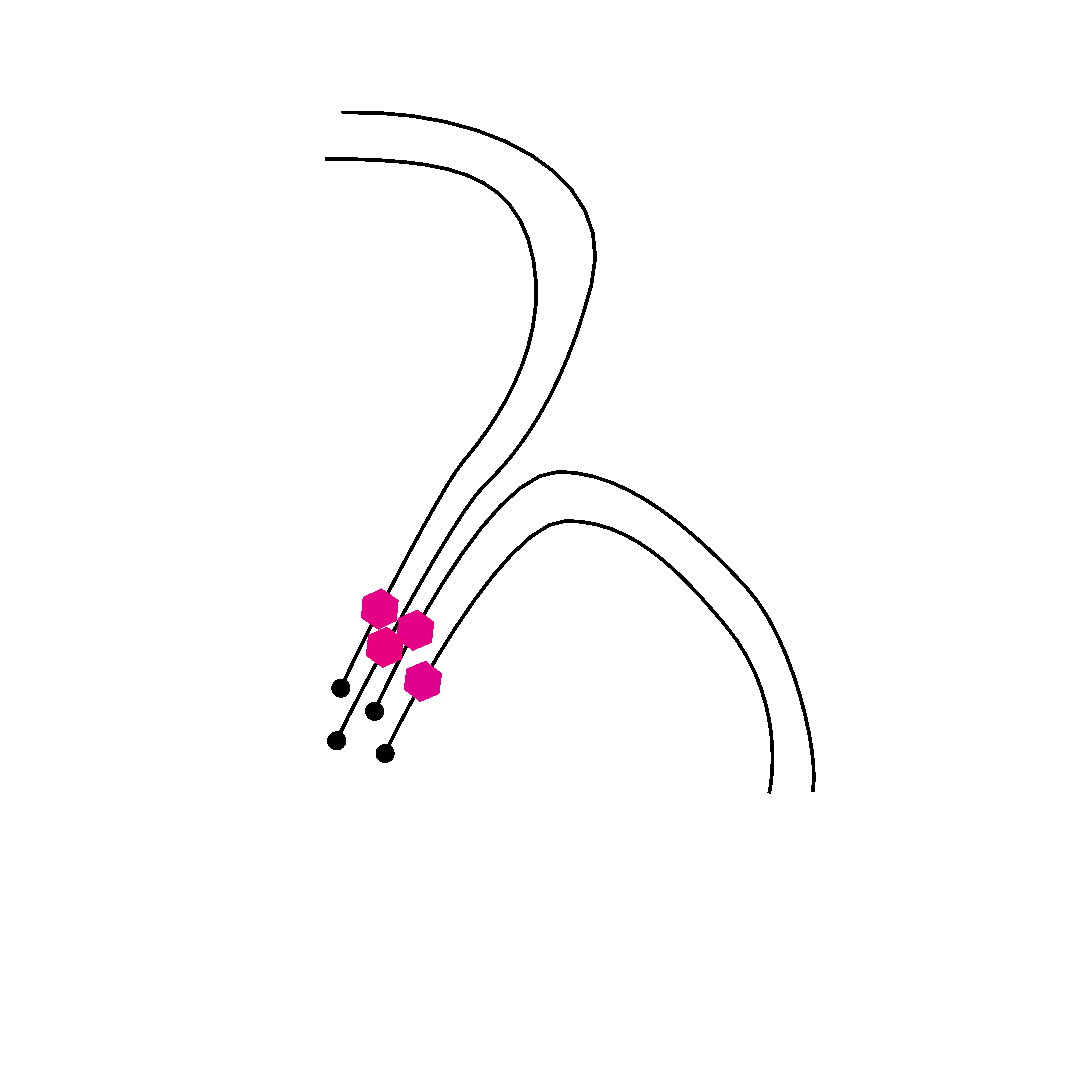}
    \caption{Four-Split; $\numrobots\!=\!4$, $\numactors\!=\!4$}
    \label{subfig:four_split}
  \end{subfigure} &
  \begin{subfigure}[b]{0.24\linewidth}
    \centering
      \includegraphics[height=\figureheight, trim={100 100 100 30}, clip]{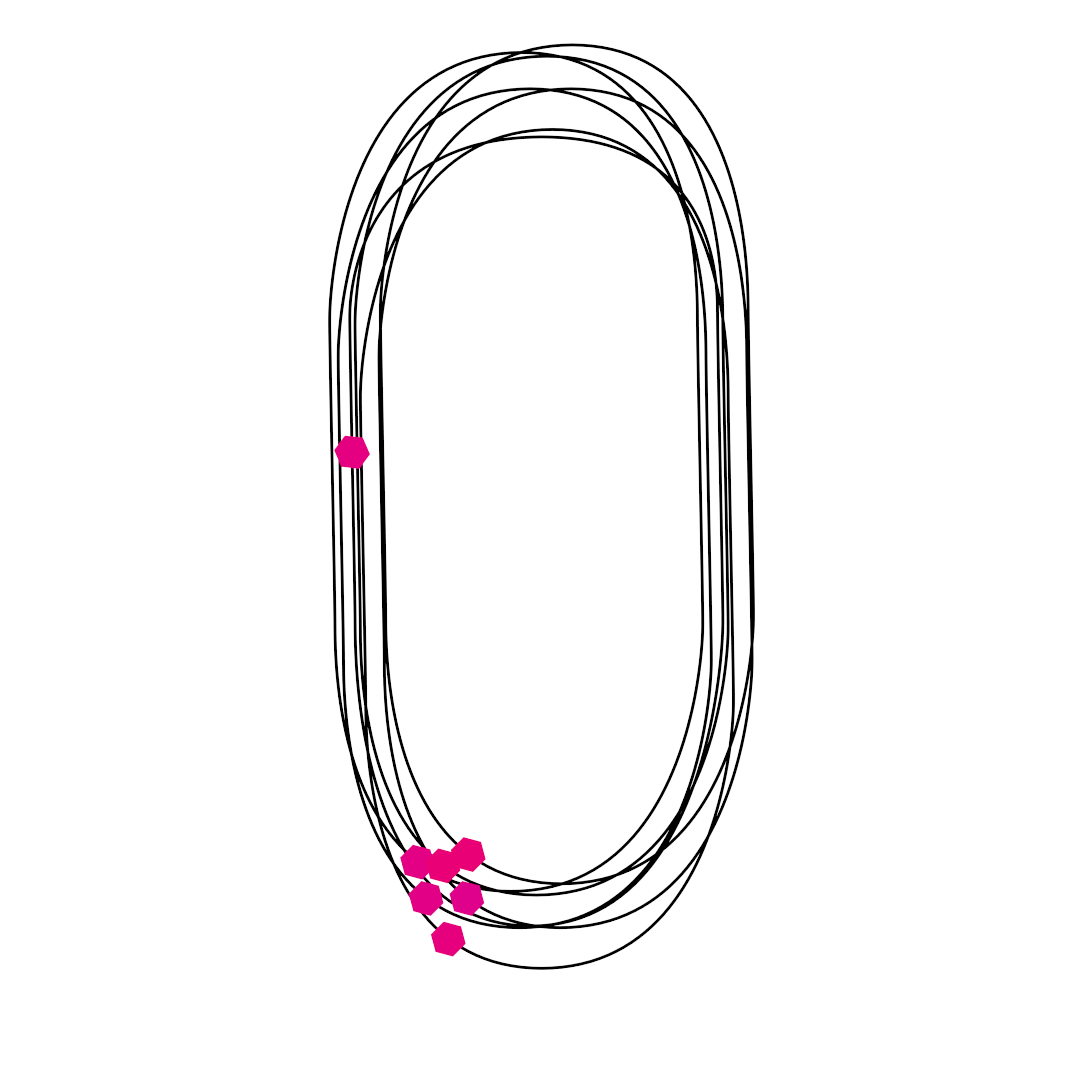}
    \caption{Priority-Runners; $\numrobots \!\!=\! 5$, $\numactors \!\!=\! 7$}
    \label{subfig:priority_runners}
  \end{subfigure} \\
  \begin{subfigure}[b]{0.24\linewidth}
    \centering
      \includegraphics[height=\figureheight, trim={100 200 250 200}, clip]{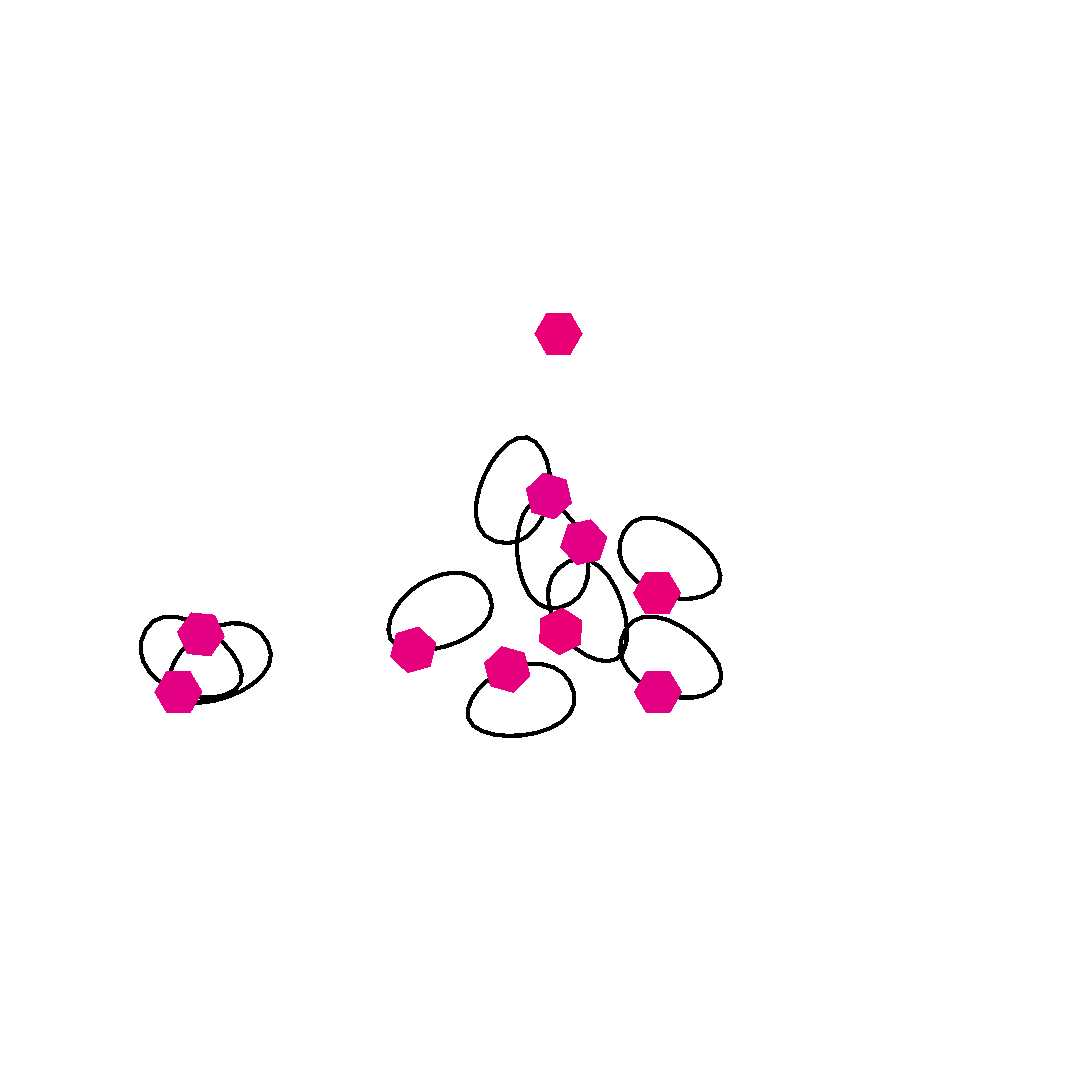}
    \caption{Priority-Speaker;\,$\numrobots\!\!=\!\!5$,\,$\numactors\!\!=\!10$}
    \label{subfig:priority_speaker}
  \end{subfigure} &
  \begin{subfigure}[b]{0.24\linewidth}
    \centering
      \includegraphics[height=\figureheight, trim={90 120 90 120}, clip]{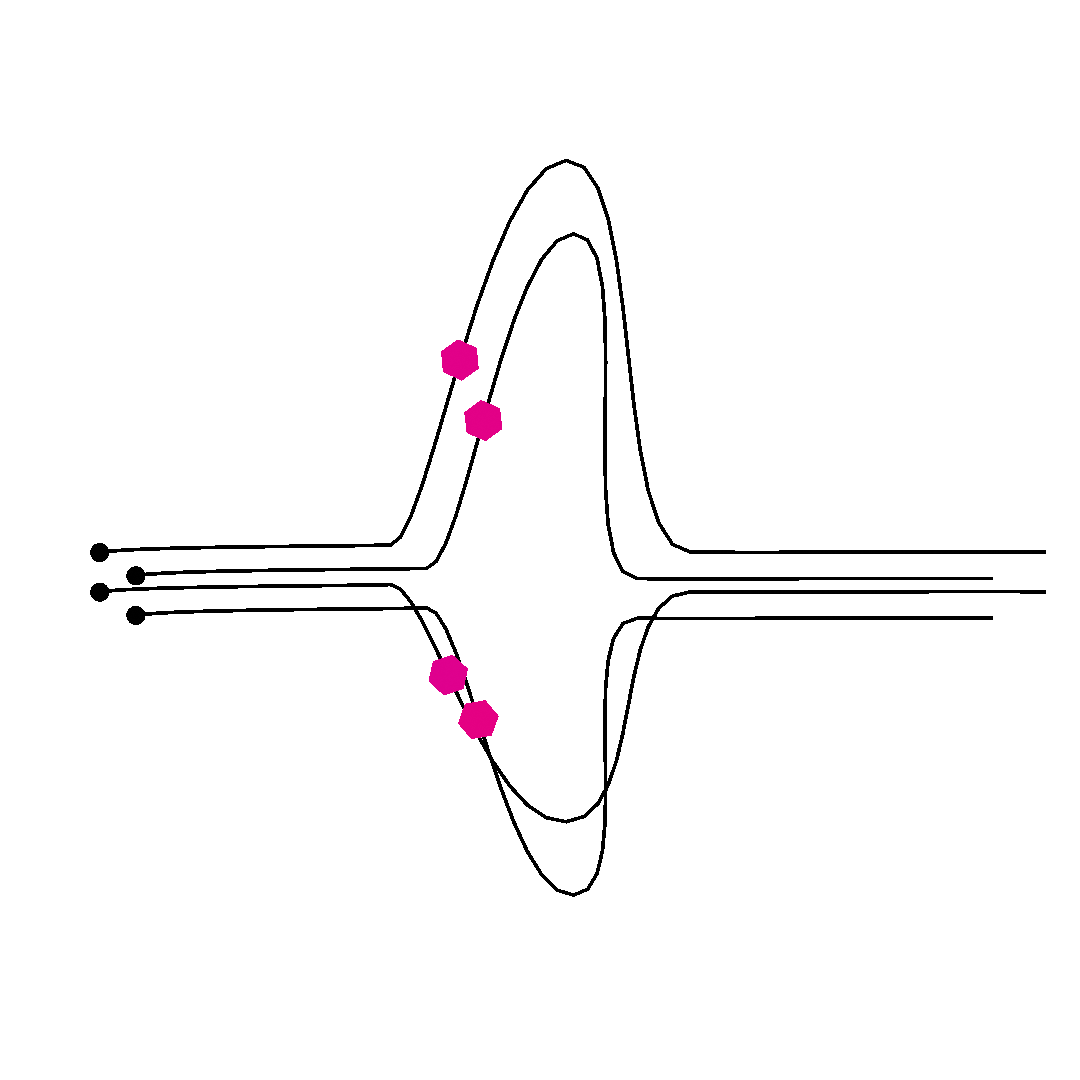}
    \caption{Split-and-Join; $\numrobots\!=\!4$, $\numactors\!=\!4$}
    \label{subfig:split_and_join}
  \end{subfigure} &
  \begin{subfigure}[b]{0.24\linewidth}
    \centering
      \includegraphics[height=\figureheight, trim={100 130 100 130}, clip]{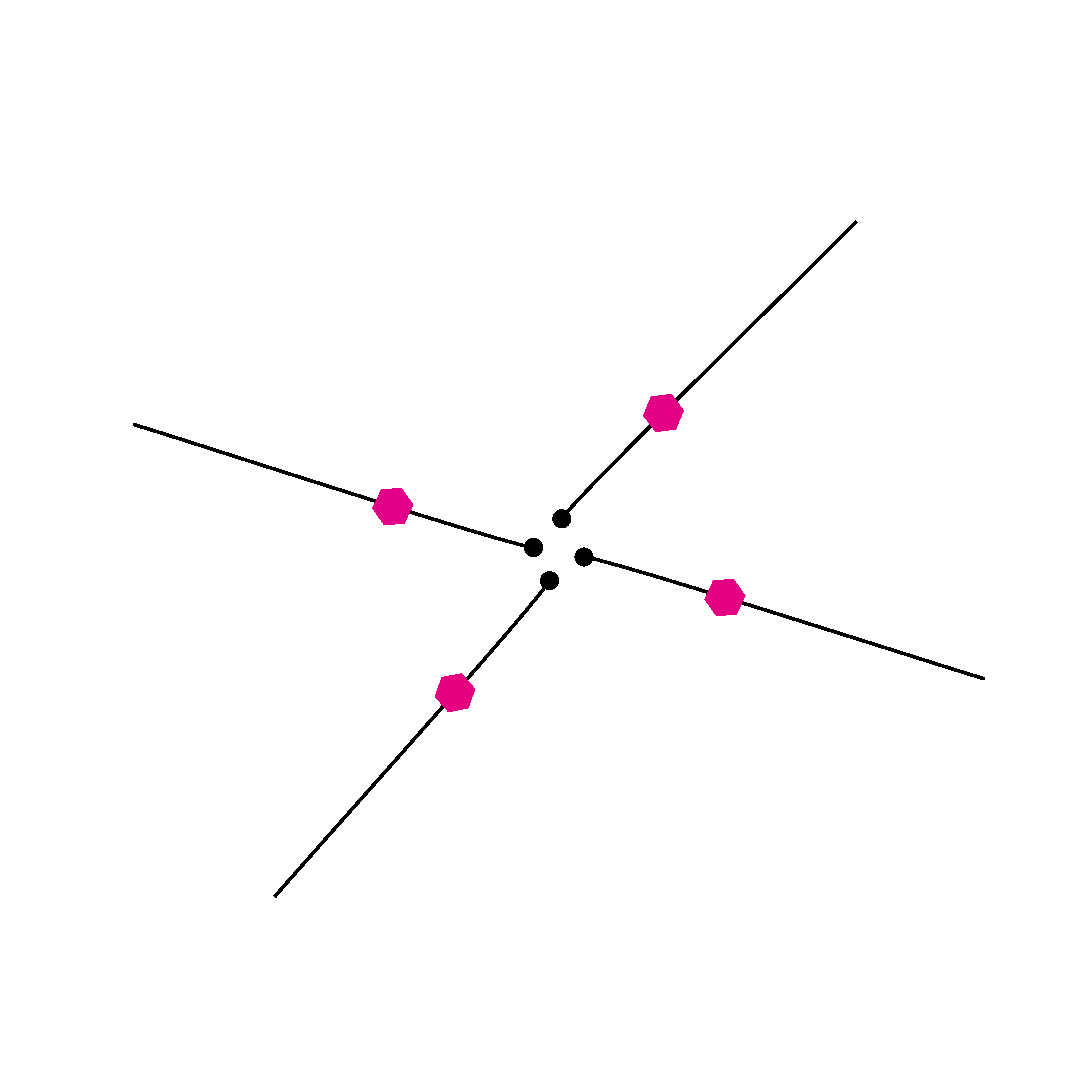}
    \caption{Spreadout-Group;\,$\numrobots\!\!=\!4$,\,$\numactors\!\!=\!4$}
    \label{subfig:spreadout_group}
  \end{subfigure} &
  \begin{subfigure}[b]{0.24\linewidth}
    \centering
      \includegraphics[height=\figureheight, trim={60 40 60 40}, clip]{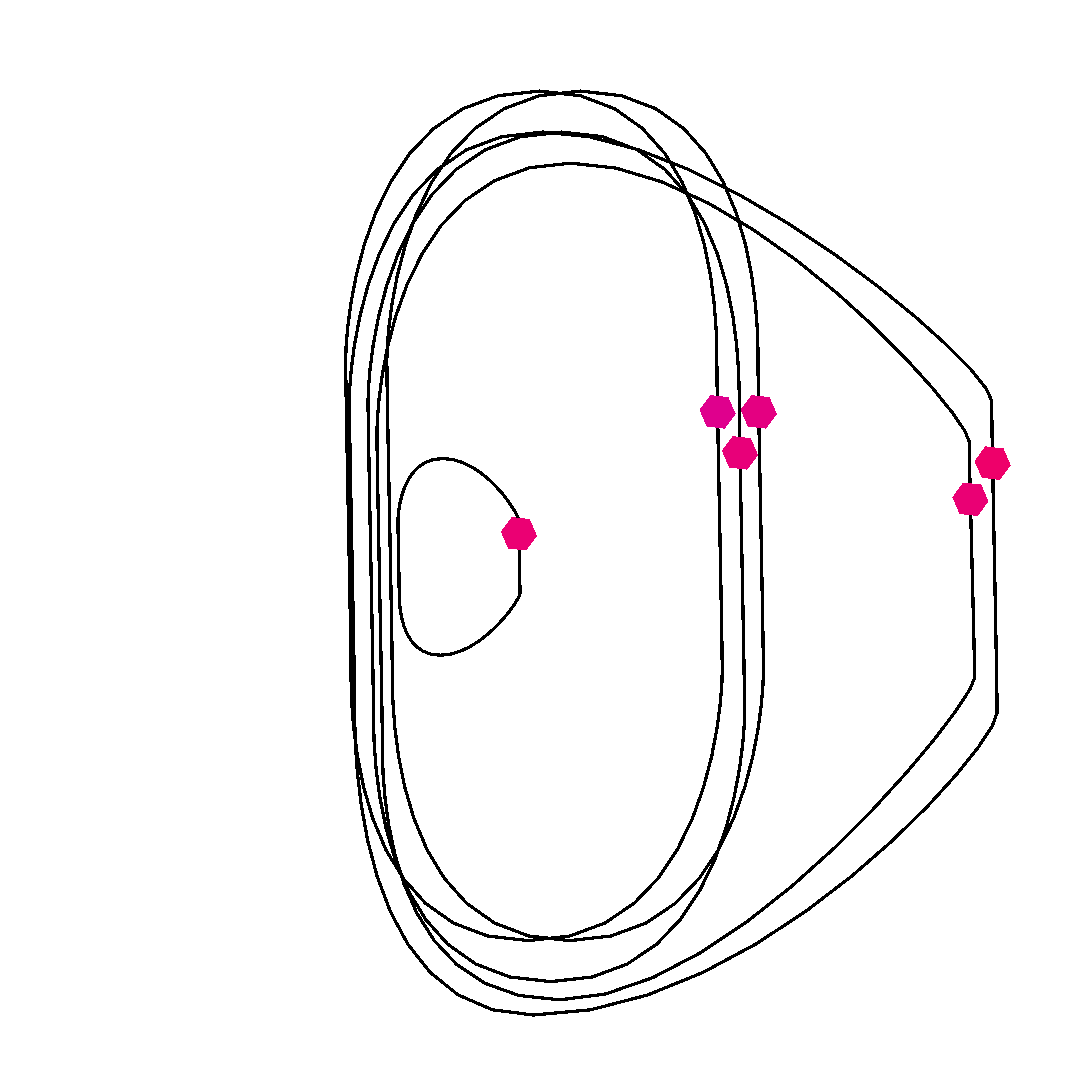}
    \caption{Track-Runners; $\numrobots\!=\!5$, $\numactors\!=\!6$}
    \label{subfig:track_runners}
  \end{subfigure}
\end{tabular}
\caption{
  Summary of experiment scenarios: Each scenario is shown from a top down view
  with actors as magenta hexagons. The actor paths are demarcated by black
  lines, and dots mark the
  starting positions. $\numrobots$ and $\numactors$ refer to the number of robots, and the number of actors respectively.
  All actors (and faces) have the same priority $w_f\!=\!1$
  (see \eqref{eq:view_quality})
  except in
  (\subref{subfig:priority_runners})
  the lead runner has $w_f\!=\!10$ and in
  (\subref{subfig:priority_speaker})
  the stationary actor (the speaker) has $w_f\!=\!5$.
}
\label{fig:scenario_summary}
\end{figure*}



\subsubsection{Common parameters and experimental setup}
Several aspects of each scenario are consistent across evaluations with
each planner.
The robots' initial positions are chosen randomly but are shared across
each planning approach.\footnote{%
  The formation planner is an exception as we specify the robot positions
  directly based on the desired formation and ignore the motion model.
}
The number of robots, the actor trajectories,
the time horizon, and grid resolution are also specified.

\subsection{Evaluation and comparison of view rewards}
\label{sec:evaluation}
\subsubsection{Planner reward (SRPPA) evaluation}

First we evaluate planners with respect to the approximated SRPPA objective
($\objective$, Sec.~\ref{sec:objective}).
By evaluating planner performance against the objective directly, we can
ascertain how well our planners perform compared to the baselines
solving the optimization problem we study \eqref{eq:submodular_maximization}.
However, this approximation does not account for challenges such as occlusions.

\subsubsection{Rendering (Image) evaluation}
\label{sec:rendering_evaluation}

For real world scenarios, consideration of the effect of occlusion on the quality of an
image is important.
For instance, when filming a sporting event, a filmmaker may wish
to choose camera views that are close to the subject, that maximize the
clarity of a view, or include several different players.
In such situations, overlap and occlusions may compromise the quality of a
view.
Although none of the approaches we consider account for visual occlusion, we
wish to quantify the impact of occlusions and other inaccuracies on our results.

We evaluate each planner with a rendering based approach built on the
3D animation software, Blender~\cite{blender}.
This enables us to directly compute pixel densities over the surfaces of the
actors.
By comparing our approximation of SRPPA to the direct rendering-based
evaluation, we can gauge the impact of occlusions and approximations on the
results.

Specifically, we map planner outputs to 3D camera movements, and we
render views of the scene with Blender. 
The 3D cameras are animated to match the position and orientation of the robots
and are tilted down by a fixed angle $\declination$ as in Fig.~\ref{fig:scene}.
We obtain an image sequence for each robot's camera and compute the SRPPA by
counting pixels on each (uniquely colored) face.
Examples of camera views can be seen in Fig.~\ref{fig:qualitative}.
The reward for a face with area $A$ which robots observe with $P$ pixels is
$A\sqrt{P/A}$, and we sum reward over all faces.
This modifies our approach for planning \eqref{eq:view_quality}
by dropping weighting for prioritization and view-centering.

\begin{figure*}
\centering
\setlength{\tabcolsep}{1pt}
\setlength{\figureheight}{1.2in}
      \resizebox{\linewidth}{!}{
      \input{figures/qualtiative}
      }
      \caption{Trajectories from \emph{Multi-Round Greedy} on
        the \emph{cross-mix} scenario:
        Each uniquely colored circle represents an actor, and
        pairs initially have similar colors.
        At the point of crossing, two out of the three groups swap
        partners. Our coordination scheme naturally handles the complex actor
        movement and produces good view diversity. Shown below are the four
        camera views used in the rendering based evaluation.
        Each face of each actor has a unique color.
      }%
      \label{fig:qualitative}
\end{figure*}

\begin{figure*}
  \setlength{\subfiguresheight}{1.75in}
  \begin{subfigure}[b]{0.25\linewidth}
    \setlength{\figurewidth}{1.1\linewidth}
    \setlength{\figureheight}{\subfiguresheight}
    \centering
    \scriptsize
    \input{figures/time_series/cross_mix.tex}
    \vspace{-1.2em}
    \caption{Cross-Mix}
    \label{subfig:evaluation_cross_mix}
  \end{subfigure}
  \begin{subfigure}[b]{0.25\linewidth}
    \setlength{\figurewidth}{1.1\linewidth}
    \setlength{\figureheight}{\subfiguresheight}
    \centering
    \scriptsize
    \input{figures/time_series/track_runners.tex}
    \caption{Track-Runners}
    \label{subfig:evaluation_track_runners}
  \end{subfigure}
  \begin{subfigure}[b]{0.5\linewidth}
    \setlength{\figurewidth}{0.55\linewidth}
    \setlength{\figureheight}{\subfiguresheight}
    \centering
    \scriptsize
\begin{tikzpicture}

\definecolor{color0}{rgb}{0.12156862745098,0.466666666666667,0.705882352941177}
\definecolor{color1}{rgb}{1,0.498039215686275,0.0549019607843137}
\definecolor{color2}{rgb}{0.172549019607843,0.627450980392157,0.172549019607843}
\definecolor{color3}{rgb}{0.83921568627451,0.152941176470588,0.156862745098039}
\definecolor{color4}{rgb}{0.580392156862745,0.403921568627451,0.741176470588235}

\begin{axis}[
height=\figureheight,
legend cell align={left},
legend style={
  fill opacity=0.8,
  draw opacity=1,
  text opacity=1,
  at={(0.97,0.03)},
  anchor=south east,
  draw=white!80!black
},
tick align=inside,
tick pos=left,
width=\figurewidth,
x grid style={white!69.0196078431373!black},
xlabel={Time-Step},
xmin=-2.95, xmax=83.95,
xtick style={color=black},
y grid style={white!69.0196078431373!black},
ymin=0.587095686573909, ymax=3.35116003580286,
ytick style={color=black}
]
\addplot [semithick, color0]
table {%
1 3.18102287657771
2 3.17798882330325
3 3.17892093986439
4 3.18057615030845
5 3.18203341532421
6 3.18316768119036
7 3.18419117269964
8 3.18497893086666
9 3.18562782980565
10 3.18289548403138
11 3.17918753010833
12 3.1752546092339
13 3.17110874412981
14 3.16674687103723
15 3.16055123597085
16 3.15643463018813
17 3.12750177202486
18 3.08910622665438
19 3.03863993480468
20 2.95610315214572
21 2.87984642453867
22 2.79763322658397
23 2.28984797510919
24 2.08132488390251
25 2.05587509493385
26 2.02926228327589
27 1.99895093749214
28 1.96592748017348
29 1.93052408928617
30 1.78589948622022
31 1.76974886829292
32 1.77168284043978
33 1.77038346275473
34 1.75781683106163
35 1.72906915114529
36 1.66860302660421
37 1.59791260312193
38 1.61348644716671
39 1.60670082065905
40 1.61816192964983
41 1.67171182527775
42 1.68741648615367
43 1.69554141025194
44 1.70278036631299
45 1.70503221842336
46 1.7352147179995
47 1.90633385696104
48 1.94935285727185
49 1.98865716729314
50 2.02457347927066
51 2.05349283672076
52 2.21251825820875
53 2.53446510727235
54 2.79248408424968
55 2.87798022474903
56 2.96902031034629
57 3.02417450952055
58 3.06608973920251
59 3.08792584220809
60 3.09609753139407
61 3.10704380690025
62 3.10979235430676
63 3.11679213801179
64 3.1226618570984
65 3.12836190920519
66 3.13389729392517
67 3.13926230115232
68 3.14444856627984
69 3.14945118935237
70 3.15426154660407
71 3.15887179399557
72 3.16327085505282
73 3.16494424463051
74 3.16416516648325
75 3.16323876719602
76 3.16214083400701
77 3.16090377918637
78 3.15948940618389
79 3.15791141207046
80 3.15613477482894
};
\addplot [semithick, color1]
table {%
1 0.712734975175225
2 1.35006629957796
3 1.69228357349857
4 2.16545564551589
5 2.50259908969362
6 2.64336387264365
7 2.86031980628274
8 3.03921240071115
9 3.21983264788995
10 3.22552074720154
11 3.21569175740728
12 3.22033272811761
13 3.21640879011474
14 3.20327609612935
15 3.20109734190258
16 3.19732236627968
17 3.16598569329223
18 3.11499587037993
19 3.12803099337136
20 3.04602741432363
21 3.03901899646544
22 2.98748133650729
23 2.91947141390466
24 2.85422285301148
25 2.79883327316528
26 2.72298023345018
27 2.66644961000552
28 2.61684404454569
29 2.57779727841146
30 2.54455972076334
31 2.503356657818
32 2.46440023000251
33 2.4331893279139
34 2.41062291590259
35 2.33947946788664
36 2.34226531479012
37 2.35023571306336
38 2.31287980243488
39 2.31156014432292
40 2.29869900986399
41 2.32681213468667
42 2.3679799894655
43 2.40956375631269
44 2.44183639458381
45 2.48409208551833
46 2.51398716531531
47 2.55535759690759
48 2.59718021027342
49 2.64482121925657
50 2.69816861041591
51 2.73285897711692
52 2.79535572287023
53 2.85925016897214
54 2.99105993266021
55 3.03732011764758
56 3.08577556409477
57 3.13278451893591
58 3.11213643056978
59 3.14723413725656
60 3.17708363653067
61 3.17073958361501
62 3.17695863035917
63 3.17388090265586
64 3.19187386736257
65 3.19217689492258
66 3.19415697564675
67 3.19750469204926
68 3.19401136012323
69 3.20074004424473
70 3.19607073915278
71 3.20360222366616
72 3.19415934182929
73 3.20549324095326
74 3.1948004925739
75 3.20556913280062
76 3.20752228826884
77 3.19149998241132
78 3.16363039779225
79 3.11479436934431
80 3.1129520280551
};
\addplot [semithick, color2]
table {%
1 0.712734975175225
2 1.36093571089624
3 1.70737331847494
4 2.09174919297493
5 2.51207204519291
6 2.57556319172051
7 2.73302482584677
8 2.79057960833887
9 2.81265508638663
10 2.85240838558528
11 2.85669554120901
12 2.83926533330027
13 2.82764202384644
14 2.8417690622078
15 2.84766008033023
16 2.76826589800683
17 2.74312791290799
18 2.83626028139398
19 2.71047437081097
20 2.73935380114048
21 2.58010114949336
22 2.48483065817505
23 2.34577031076853
24 2.3243733469602
25 2.19868624852214
26 2.16871846816037
27 2.1751180166657
28 2.1703213521126
29 2.13860568510872
30 2.11486263171258
31 2.14649955560922
32 2.13311672133
33 2.07880760312192
34 2.07065349737364
35 2.03734564281062
36 1.98771832952391
37 2.00234766908041
38 1.96347084371065
39 1.92757406213854
40 1.87276119177594
41 1.79511839579368
42 1.83108614700273
43 1.82209558395467
44 1.88621442505594
45 1.87830232859113
46 1.83264308070284
47 1.8337001408831
48 1.8894041309882
49 1.87098291042596
50 1.83666640421069
51 1.89604985925585
52 1.92691089455682
53 1.89844446148062
54 2.14386484808078
55 2.30444070386759
56 2.15941373811304
57 2.33020963930817
58 2.44005430635887
59 2.48638462182282
60 2.49648226268921
61 2.41249186970097
62 2.44240765615123
63 2.50556533672444
64 2.56215437752701
65 2.43310914409967
66 2.46634774818995
67 2.49896297130341
68 2.56452722976249
69 2.57166402830876
70 2.44426844545488
71 2.47629012482112
72 2.52528038688056
73 2.57667707673453
74 2.54821157079869
75 2.47360151776793
76 2.48615790614017
77 2.57691615329854
78 2.50459971318651
79 2.53129868224907
80 2.53028542483736
};
\addplot [semithick, color3]
table {%
1 0.712734975175225
2 1.34786137001964
3 1.68194371731512
4 2.16545564551589
5 2.50259908969362
6 2.60872760863384
7 2.82092638032573
8 2.88322380538612
9 2.90252742359647
10 2.90036300629924
11 2.89790125861458
12 2.89179698031844
13 2.89612091981158
14 2.92157194259067
15 2.91153254239235
16 2.89204999762669
17 2.73439773271145
18 2.58947270395259
19 2.44505915887308
20 2.3245318481285
21 2.28649652850834
22 2.233510503231
23 2.18940159969856
24 2.13239265954958
25 2.07734707696752
26 2.02907744908536
27 1.97226680840426
28 1.93043505959322
29 1.876680330106
30 1.85168319485164
31 1.80626659021266
32 1.7769094979813
33 1.72117925827274
34 1.69266185997182
35 1.6258298301332
36 1.65002180353391
37 1.69675614810766
38 1.62846595083494
39 1.64266211081517
40 1.64504077521842
41 1.6746618006599
42 1.66614874166395
43 1.72714750554972
44 1.75782522682327
45 1.80914432598124
46 1.84246500578151
47 1.88070816995284
48 1.91724240673863
49 1.95760371149012
50 2.00144439605875
51 2.04642965378292
52 2.09508997467916
53 2.13961613076431
54 2.16342937257263
55 2.24311856354149
56 2.31735340089724
57 2.36129893365392
58 2.30680460183527
59 2.35260802348135
60 2.36858481601547
61 2.35750025179847
62 2.40409735018518
63 2.39722745788966
64 2.41202861999024
65 2.39828221281238
66 2.41584235861548
67 2.40655104345362
68 2.41812727809097
69 2.41365559981696
70 2.41874083449015
71 2.41945539225036
72 2.41753919145301
73 2.42382186380613
74 2.41436572403087
75 2.42663794506833
76 2.41170266601625
77 2.42779479855717
78 2.41851170824063
79 2.42717342548923
80 2.42442649389534
};
\addplot [semithick, color4]
table {%
1 0.712734975175225
2 1.36438069928286
3 1.73998862064878
4 2.36320159945577
5 2.78300253683979
6 2.9701160734125
7 3.14410905903747
8 3.17357584248476
9 3.13916187566711
10 3.13550527202305
11 3.20256913076624
12 3.22033272811761
13 3.21640879011474
14 3.21369701159224
15 3.2060986005345
16 3.21032350444268
17 3.17839825189871
18 3.11810731561803
19 3.12803099337136
20 3.04712456981628
21 3.04046363019158
22 2.98748133650729
23 2.91584844208076
24 2.85414683012487
25 2.79981753037178
26 2.74446940909027
27 2.69144502609808
28 2.63975593378594
29 2.59874788362081
30 2.5649930500644
31 2.52630798983406
32 2.49307415903989
33 2.46573528869251
34 2.44264144015504
35 2.41911905886561
36 2.39253063289763
37 2.40363367784366
38 2.37821739753166
39 2.37540938216569
40 2.38859392543782
41 2.39876510882083
42 2.43870263196983
43 2.47528508498603
44 2.49465961578672
45 2.51998448182345
46 2.56187071007149
47 2.60368805609989
48 2.63835617286431
49 2.68733733357247
50 2.73214027063061
51 2.78599279230443
52 2.84448159696019
53 2.89732220341724
54 2.96110246673078
55 3.01942586360421
56 3.08931224537434
57 3.11813768192437
58 3.139557738566
59 3.15846421126631
60 3.17896700749178
61 3.17100647257002
62 3.17697045262582
63 3.18872948592032
64 3.19119763191544
65 3.18573929197846
66 3.19415697564675
67 3.19750469204926
68 3.19452607450328
69 3.20074004424473
70 3.19876244311175
71 3.20360222366616
72 3.20226859183676
73 3.20549324095326
74 3.20510755055382
75 3.20210081355711
76 3.2044522235353
77 3.19070674438871
78 3.16363039779225
79 3.11479436934431
80 3.1129520280551
};
\end{axis}

\end{tikzpicture}
\begin{tikzpicture}

\definecolor{color0}{rgb}{0.12156862745098,0.466666666666667,0.705882352941177}
\definecolor{color1}{rgb}{1,0.498039215686275,0.0549019607843137}
\definecolor{color2}{rgb}{0.172549019607843,0.627450980392157,0.172549019607843}
\definecolor{color3}{rgb}{0.83921568627451,0.152941176470588,0.156862745098039}
\definecolor{color4}{rgb}{0.580392156862745,0.403921568627451,0.741176470588235}

\begin{axis}[
height=\figureheight,
legend cell align={left},
legend style={
  fill opacity=0.8,
  draw opacity=1,
  text opacity=1,
  at={(0.03,0.03)},
  anchor=south west,
  draw=white!80!black
},
tick align=inside,
tick pos=left,
width=\figurewidth,
x grid style={white!69.0196078431373!black},
xlabel={Time-Step},
xmin=-2.95, xmax=83.95,
xtick style={color=black},
y grid style={white!69.0196078431373!black},
ylabel={Evaluation (Image)},
ymin=67.9678140073063, ymax=450.761712534841,
ytick style={color=black}
]
\addplot [semithick, color0]
table {%
1 341.75168834256
2 345.289674699098
3 342.277814739666
4 344.637237404044
5 346.736974168982
6 345.543977610794
7 346.509741175853
8 348.516711376978
9 346.530351781856
10 349.082857495729
11 346.954963957021
12 344.977570869251
13 345.014566477783
14 341.153042356968
15 342.120938137284
16 347.64243670595
17 364.67263112658
18 352.275204546014
19 339.366594090913
20 333.223186437124
21 330.641918129088
22 311.221565436269
23 287.436218172838
24 272.584810030909
25 263.017523280376
26 260.903388596341
27 249.525823046126
28 242.737927456389
29 224.170195351276
30 215.558198489044
31 207.999556418801
32 199.739325733278
33 201.9745426917
34 199.850485484298
35 194.181788462242
36 185.238039778975
37 173.092165098348
38 177.711599399017
39 182.160375938857
40 182.241874798346
41 194.950999025806
42 199.964023952767
43 207.36528013713
44 210.43351096557
45 211.487469378797
46 244.462380330466
47 255.596378823803
48 261.824317228854
49 281.422588026677
50 282.535911743782
51 306.942992038178
52 332.960538980494
53 352.294652429334
54 361.946371547979
55 367.458376034948
56 374.564074735718
57 395.449411930549
58 387.023269829375
59 375.157275236013
60 367.75290810901
61 361.50872803326
62 361.971587554655
63 364.378206585604
64 366.223436345845
65 365.317643936776
66 367.494077826511
67 364.439337507977
68 362.756688563733
69 366.032213852175
70 366.654765472489
71 365.632386398579
72 369.006806520741
73 365.404623995287
74 366.852110184569
75 363.822391391234
76 366.100336681659
77 367.164678563074
78 367.686386526412
79 365.513413491985
80 85.3675366676488
};
\addplot [semithick, color1]
table {%
1 131.674122980344
2 141.073615248865
3 213.867934061439
4 261.660720555037
5 297.271122019381
6 283.319901165099
7 321.213857116357
8 363.031169729873
9 357.501217685276
10 370.774760895001
11 351.764170530614
12 365.513976543558
13 355.535120953548
14 356.281519529213
15 355.197442522914
16 349.63129608369
17 405.219403881946
18 404.22523077824
19 426.474361960043
20 391.028099048845
21 393.200716520242
22 395.31855728698
23 383.972542998556
24 376.208801356058
25 359.457063692839
26 357.477115287731
27 350.982755922659
28 356.754393799568
29 349.762310563452
30 346.029965282147
31 334.957228845019
32 329.191831840829
33 332.990308917012
34 292.718452716416
35 293.130771098276
36 284.66392404457
37 273.312410687058
38 278.725915726003
39 234.878938608838
40 264.140928351087
41 259.789580031556
42 261.242955545495
43 276.1850914618
44 275.132107574054
45 295.786799102012
46 295.610492228281
47 301.69820025106
48 285.89803218474
49 311.617341118162
50 306.494221023692
51 316.334234333197
52 366.951410643135
53 400.779463944322
54 399.321974411734
55 406.015743398354
56 399.163460818608
57 376.974715876079
58 402.137408387301
59 395.183187321475
60 392.536385676253
61 377.026407356268
62 373.092811559785
63 385.437879141694
64 397.935219629718
65 385.22353202257
66 387.634850831863
67 375.890035481429
68 382.334067238111
69 369.67294484233
70 382.048419217544
71 367.872212336105
72 380.898630474972
73 367.347051257011
74 377.51192485038
75 378.984159972703
76 375.368593223757
77 378.293554507341
78 371.176000433039
79 368.696425263358
80 85.3675366676488
};
\addplot [semithick, color2]
table {%
1 124.73517173732
2 143.453720855625
3 208.597787786443
4 271.799178419682
5 302.486182055223
6 299.117786087455
7 318.665787315287
8 326.670769822308
9 336.443498782199
10 326.979887214113
11 328.253696113286
12 332.138602038304
13 324.955400529569
14 330.154054384777
15 323.440489334434
16 323.23315048084
17 374.310689828772
18 376.746957426708
19 397.595952757453
20 403.8098981721
21 389.874398190056
22 372.399415974531
23 357.790234694256
24 334.071616640598
25 325.654372662435
26 325.185927213403
27 328.445303807172
28 313.118429606973
29 314.336419542113
30 314.841067743294
31 311.030171979141
32 303.246678154531
33 306.362040555705
34 306.845066335015
35 287.156452971816
36 278.258895778148
37 259.789300779483
38 258.977667344047
39 264.015616719248
40 238.387563649466
41 234.652167063684
42 231.312866200501
43 235.18048849431
44 242.500917414615
45 233.187391614711
46 228.53514722615
47 247.306200448609
48 245.550405637694
49 240.98118255887
50 223.896880374057
51 242.19616636469
52 249.524736951709
53 285.937766953304
54 286.06041742405
55 232.854324470619
56 246.082076835998
57 253.00340564819
58 275.992591243733
59 282.205424077383
60 264.957695546655
61 263.100867940412
62 273.49136100595
63 290.921829277006
64 262.367829128173
65 258.517489894962
66 273.059248548279
67 286.933303776044
68 302.45965257368
69 251.461533407631
70 262.940924789587
71 283.886954162014
72 295.512916099623
73 292.740301103542
74 256.580444806163
75 268.641903847401
76 281.769988464386
77 254.838463500996
78 265.818987902931
79 277.583535562771
80 85.3675366676488
};
\addplot [semithick, color3]
table {%
1 131.676830701985
2 136.537356395392
3 213.867934061439
4 261.660720555037
5 288.484431053737
6 316.267204197506
7 330.371788647584
8 325.794224447679
9 329.700331029661
10 339.838068523662
11 324.698531315259
12 338.384224561011
13 343.097367238954
14 333.77577384709
15 334.420053687662
16 325.81852998313
17 324.491599946468
18 305.211489497982
19 297.660209727428
20 305.261370806761
21 313.863949990188
22 312.413568345163
23 301.363906670344
24 297.936349742059
25 290.623296799238
26 282.450222487454
27 280.277108690318
28 272.21371879036
29 269.337872867291
30 264.754457738489
31 257.490235825448
32 254.704075292587
33 251.086276166266
34 174.883760322716
35 175.470394350737
36 180.604986798149
37 165.579975416353
38 162.419227443925
39 133.941826310616
40 155.490599872946
41 135.738096887736
42 150.505442644185
43 138.595388734452
44 157.674154642691
45 184.565882202139
46 161.524864372152
47 184.730423924643
48 172.125003642551
49 197.828157778577
50 180.555412697761
51 201.537455444572
52 188.977554493297
53 212.519232738384
54 201.03872336866
55 245.331138276433
56 267.309102064291
57 251.39220908268
58 290.911146901773
59 292.939789819742
60 287.202260158589
61 280.626081352647
62 257.222746992687
63 283.837424858594
64 289.283992417741
65 283.392224881863
66 290.994422239212
67 279.368913308433
68 289.168315186605
69 269.407973254987
70 287.194660708351
71 265.215795011724
72 284.232659620031
73 262.031467646977
74 279.755427204008
75 282.776093236866
76 264.663995076061
77 277.618388756023
78 255.67943306531
79 269.95613308477
80 85.3675366676488
};
\addplot [semithick, color4]
table {%
1 117.753616955784
2 140.096734447092
3 210.955305760565
4 291.290558884373
5 322.374252475663
6 346.308912385035
7 358.934307885092
8 365.569401666408
9 362.228633013719
10 365.491213130088
11 351.764170530614
12 365.513976543558
13 356.51006636287
14 358.483032810313
15 360.331600385888
16 366.06317577383
17 404.152662465961
18 404.22523077824
19 433.361989874499
20 398.779543230203
21 393.200716520242
22 389.518417745015
23 377.415805281873
24 376.996648499028
25 370.710844490653
26 369.733518933937
27 360.72351482812
28 359.218077675397
29 358.157893349394
30 351.393321439998
31 342.756414684638
32 341.671085845359
33 338.079989639304
34 336.866542594071
35 329.344496512218
36 326.409492532596
37 324.085396957927
38 337.038788721945
39 333.026667207785
40 299.197609802203
41 325.699993232918
42 341.187685996075
43 323.064053555636
44 315.140941566907
45 333.071780358996
46 347.14165112838
47 343.385277512495
48 335.725750253644
49 347.790153963553
50 346.696702351017
51 359.335165745368
52 344.401870869938
53 356.953427899225
54 374.293546512503
55 388.625585011314
56 388.321875506802
57 411.540823204846
58 409.135210120553
59 401.165670793659
60 398.519025309678
61 377.045324620681
62 390.301807664127
63 389.802355027771
64 387.742466205828
65 385.22353202257
66 387.634850831863
67 373.670621503679
68 382.334067238111
69 381.949305037196
70 382.048419217544
71 372.106689292694
72 380.898630474972
73 378.4392533452
74 376.060155687147
75 373.468467966419
76 375.388819236497
77 378.293554507341
78 371.176000433039
79 368.696425263358
80 85.3675366676488
};
\end{axis}

\end{tikzpicture}
    \caption{Split-and-Join}
    \label{subfig:evaluation_split_and_join}
  \end{subfigure}
  \caption{View rewards plotted for each planner for select scenarios:
    In all cases, peaks correspond to when actors are near each other and
    troughs to when they are far apart.
    Plots show SRPPA as approximated by \eqref{eq:pixelterm},
    except for (\subref{subfig:evaluation_split_and_join})
    \emph{split-and-join}
    we also compare to the \emph{image-based} evaluation
    (Sec.~\ref{sec:rendering_evaluation}).
  }
  \label{fig:view_rewards}
\end{figure*}

\begin{table}
    \resizebox{\linewidth}{!}{
    \tiny
    \setlength{\tabcolsep}{3pt}
    \begin{filecontents*}{figures/data/data_summary.csv}
,formation_SRPPA,formation_Image,assignment_SRPPA,assignment_Image,myopic_SRPPA,myopic_Image,greedy_SRPPA,greedy_Image,multipleroundsgreedy_SRPPA,multipleroundsgreedy_Image
cluster,2.7234436614199917,29.680107648839023,2.7298211047466823,26.813965707197095,2.6991547547365644,26.715712605059466,2.795513859891797,29.23886942466766,2.7957040878099844,29.24442999434271
cross mix,3.163490245703812,38.42515868123998,2.859789408796076,37.592777481536004,3.134143831397275,38.62025403823671,3.7399592591344475,45.705230042209834,3.7483929680021078,45.87474097043691
four split,2.608882385050996,31.33496109245354,2.4427808765431203,30.639835025469594,2.0922334528612634,27.176203186372735,2.7735526717754833,34.64921343749969,2.7861199233572944,34.37347233754093
priority runners,13.58013419990595,51.488912764993785,12.022544606332733,64.5986626059968,13.754946310584117,40.99708236080249,17.771526684735612,63.27734243766802,17.840753703776098,64.24700719190533
priority speaker,4.679685651815981,36.32191200893424,4.816597699951898,44.240861519821664,4.6355016843020165,38.70840063345872,5.853612161012059,49.773946368056066,5.868484850513068,49.55690693912411
split and join,2.1045860269824077,24.47718557187843,1.837705798769804,22.633288626492067,1.7637809011993226,20.144718974591925,2.254077010109555,27.082799038473567,2.2783729449721295,28.184578101758692
spreadout group,1.4106666372157264,18.43900973728404,1.3323564408923436,18.66401598150863,1.5328121956710714,20.00041948230275,1.8864415574467772,25.706915590069194,1.8931182789139027,26.11659225535935
track runners,5.434065143273122,54.569047520465176,4.99646947451305,57.5377420350107,5.514469750145577,56.69363160521364,6.983764131681475,78.05152154796819,7.009843342583606,78.19152349921961
\end{filecontents*}
\pgfplotstabletypeset[
      every head row/.style={
        output empty row,
        before row={
          \toprule
          \textbf{Scenario}
          &
          \multicolumn{2}{c}{\textbf{Formation}}
          &
          \multicolumn{2}{c}{\textbf{Assignment}}
          &
          \multicolumn{2}{c}{\textbf{Myopic}}
          &
          \multicolumn{2}{c}{\textbf{Greedy}}
          &
          \multicolumn{2}{c}{\textbf{Multi-Round}}
          \\
          &SRPPA$\times 10^2$&Image$\times 10^3$&SRPPA&Image&SRPPA&Image&SRPPA&Image&SRPPA&Image
          \\
        },
        after row=\midrule
      },
      every last row/.style={
        after row=\bottomrule
      },
      every even column/.style={
        precision={0}
      },
      every even column/.style={
        precision={2}
      },
      display columns/0/.style={
        column type={l},
        string type,
      },
      every row 0 column 9/.style={
        postproc cell content/.style={
          @cell content/.add={$\bf}{$}
        }
      },
      every row 0 column 7/.style={
        postproc cell content/.style={
          @cell content/.add={$\bf}{$}
        }
      },
      every row 0 column 2/.style={
        postproc cell content/.style={
          @cell content/.add={$\bf}{$}
        }
      },
      every row 1 column 10/.style={
        postproc cell content/.style={
          @cell content/.add={$\bf}{$}
        }
      },
      every row 1 column 9/.style={
        postproc cell content/.style={
          @cell content/.add={$\bf}{$}
        }
      },
      every row 2 column 9/.style={
        postproc cell content/.style={
          @cell content/.add={$\bf}{$}
        }
      },
      every row 2 column 8/.style={
        postproc cell content/.style={
          @cell content/.add={$\bf}{$}
        }
      },
      every row 3 column 9/.style={
        postproc cell content/.style={
          @cell content/.add={$\bf}{$}
        }
      },
      every row 3 column 4/.style={
        postproc cell content/.style={
          @cell content/.add={$\bf}{$}
        }
      },
      every row 4 column 9/.style={
        postproc cell content/.style={
          @cell content/.add={$\bf}{$}
        }
      },
      every row 4 column 8/.style={
        postproc cell content/.style={
          @cell content/.add={$\bf}{$}
        }
      },
      every row 5 column 9/.style={
        postproc cell content/.style={
          @cell content/.add={$\bf}{$}
        }
      },
      every row 5 column 10/.style={
        postproc cell content/.style={
          @cell content/.add={$\bf}{$}
        }
      },
      every row 6 column 9/.style={
        postproc cell content/.style={
          @cell content/.add={$\bf}{$}
        }
      },
      every row 6 column 7/.style={
        postproc cell content/.style={
          @cell content/.add={$\bf}{$}
        }
      },
      every row 6 column 10/.style={
        postproc cell content/.style={
          @cell content/.add={$\bf}{$}
        }
      },
      every row 7 column 9/.style={
        postproc cell content/.style={
          @cell content/.add={$\bf}{$}
        }
      },
      every row 7 column 10/.style={
        postproc cell content/.style={
          @cell content/.add={$\bf}{$}
        }
      },
      col sep=comma,
      empty cells with={--} 
    ]{figures/data/data_summary.csv}
  }

    \caption{Complete tabular results: Each entry corresponds to the sum of view
    rewards over the planning horizon for a given scenario and planner.}%
    \label{table:results}
\end{table}

\section{Results and discussion}
\label{sec:results}


\


We are interested in observing how well maximizing our SRPPA objective achieves
the intuition laid out in the start of Sec.~\ref{sec:objective} both
qualitatively and quantitatively.
To quantify this, we compare planners based on the evaluation methods in
Sec.\ref{sec:evaluation} in terms of overall performance in
Tab.~\ref{table:results} and as a function of time in
Fig.~\ref{fig:view_rewards}.
Our planners achieve desired behaviors across all of the
scenarios and make significant performance gains against our
baselines according to the evaluation.
While all planners tend to achieve similar view quality when actors are together
as for the \textit{cluster} evaluations in
Tab.~\ref{table:results}, we see more variation in performance
with the more complex scenarios.
Figure~\ref{fig:view_rewards} showcases variation in results across three
challenging scenarios that involve unstructured actor mixing, splitting, and
joining.
In these scenes, we see that our planners achieve consistent performance gains
over the baselines, especially in moments of high target separation, such as the
middle portion of \textit{split-and-join}, or when the group motion is highly
unstructured, such as throughout \textit{track-runners}.
In those cases, we observe that the performance of the formation planner is tied
closely to whether the actor distribution is small and circular, and the
assignment planner performs well when assignments correlate with the actor
distribution.
Since our planning approach imposes few constraints on the robots' motions it is
able to adapt naturally to the complex movements in each scenario.
We also find both evaluation metrics rate our planners highly.
Table~\ref{table:results} highlights the best performing planners in each
scenario; greedy and multi-round-greedy consistently achieve the highest
scores.
Providing effective and intuitive coordination via submodular maximization
of the view reward is one of our key results.




However, the SRPPA evaluation does not consider true 3D camera perspective or
actor-actor occlusions.
Despite this, in Tab.~\ref{table:results} we see both methods of evaluation
align closely in relative performance rankings between planners.
Figure~\ref{subfig:evaluation_split_and_join} also showcases time series plots
for \textit{split-and-join} for both the SRPPA evaluation and the rendering
based evaluation.
The plots are similar in structure, and relative performance of different
planners is largely consistent.

\section{Conclusions}

In this work, we proposed a new method to plan for a team of robots to
film groups of moving actors that may execute complex scripted
trajectories such as splitting, spreading out, and reorganizing such as might
appear in sport, theater, or dance.
Filming these behaviors challenges systems based on assignment or formation
planning so we instead optimize total view quality directly.
Toward this end, we presented the SRPPA objective which is a function of pixel
densities over the surfaces of the actors, and we proved that this objective is
submodular.
As such, we proposed planning for the multi-robot team via greedy methods for
submodular optimization.
Our results demonstrate that planning via greedy submodular optimization meets
or exceeds performance of assignment and formation baselines in all scenarios.
Moreover, our approach also produces intuitive behaviors implicitly such as
splitting formations when groups spread apart or changing assignments when
actors cross or rearrange.



\bibliographystyle{IEEEtranN}
{
\scriptsize
\bibliography{./IEEEfull,micah_abbreviations,refs}
}

\clearpage

\appendices

\ifextended
\section{Proofs and monotonicities}
\label{appendix:proofs}

This appendix provides the proof of the main theoretic result in this paper,
Theorem~\ref{theorem:submodular_objective}.
On the way to that result, we will also provide some background information on
derivatives of set functions and monotonicities and a key incremental result
(Lemma~\ref{lemma:real_composition}).

\subsection{Monotonicity and derivatives of set functions}
\label{sec:monotonicity}

Section~\ref{section:background_submodularity} provided a basic background on
submodular and monotonic functions.
Our analysis will rely on the slightly more general foundation of
higher-order derivatives and monotonicities which generalize monotonicity and
submodularity~\citep{foldes2005,corah2020phd}.
Our exposition will follow the notation of our prior work
\citep[Section~3.5.1]{corah2020phd} and foundations by~\citet{foldes2005}.
We will write these higher-order monotonicities in terms of derivatives of set
functions, and such derivatives are as follows:

\begin{definition}[Derivative of a set function]
  \label{definition:set_derivative}
  The $n^\mathrm{th}$ derivative of a set function
  $\setfun : \ground \rightarrow \real$
  at $X\subseteq \ground$
  with respect to some disjoint sets $Y_1,\ldots,Y_n$
  can be written recursively as
  \begin{align*}
    &\setfun(Y_1;\ldots;Y_n|X) = \\
    &\setfun(Y_1;\ldots;Y_{n-1}|X,Y_n) - \setfun(Y_1;\ldots;Y_{n-1}|X),
  \end{align*}
  defining the base case as $\setfun(\cdot|X)=\setfun(X)$.
\end{definition}
We can then define monotonicity conditions in terms of derivatives of set
functions:

\begin{definition}[Higher-order monotonicity of set functions]
  \label{definition:set_monotonic}
A set function $\setfun$ is $m$-increasing if
\begin{align*}
    \setfun(Y_1;\ldots;Y_m|X) &\geq 0\\
    \intertext{always holds or respectively $m$-decreasing if}
    \setfun(Y_1;\ldots;Y_m|X) &\leq 0
\end{align*}

\end{definition}
\noindent
Based on this definition, monotonicity and submodularity are equivalent
to a set function being 1-increasing and 2-decreasing, respectively.

Functions of real numbers can also satisfy similar monotonicity conditions.
We will use the following definition of monotonicity: 
\begin{definition}[Monotonicity of real functions]
  \label{definition:real_monotonic}
  Consider a real function $\realfun : \real \rightarrow \real$.
  Then, $f$ is $m$-increasing if
  \begin{align*}
    \realfun^{(m)}(x) \geq 0
  \end{align*}
  or respectively $m$-decreasing if
  \begin{align*}
    \realfun^{(m)}(x) \leq 0
  \end{align*}
  where $\realfun^{(m)}$ refers to the $m^\mathrm{th}$ derivative of $\realfun$.
\end{definition}

\subsection{Modular functions}

The view reward \eqref{eq:view_quality} of the SRPPA objective initially
computes a sum based on pixel densities associated with each face
($\pixelterm$).
Set functions that can be written as a sum over weights for an input set like
so are modular.

\begin{definition}[Modular set function]
  \label{definition:modular}
  A set function  $\setfun : 2^\ground \rightarrow \real$
  is modular if it is both submodular and supermodular.
  All such functions can be written as a sum of weights and an offset
  \begin{align}
    \setfun(X) = k + \sum_{x\in X} w_x
    \label{eq:modular}
  \end{align}
  where $X\subseteq \ground$ and $w_x \in \real$.
  This is easy to see as $\setfun(c|A) = \setfun(c|B) = \setfun(c)-k = w_c$ for
  $B\subseteq A \subseteq \ground$ and $c \in \ground \setminus A$ following the
  discussion of submodularity in Sec.~\ref{section:background_submodularity}.
\end{definition}
As a consequence,
any modular set function
satisfies
\begin{align}
  \setfun(A,B) = \setfun(A) + \setfun(B)
  \label{eq:modular_additive}
\end{align}
for any disjoint sets $A$ and $B$.
A modular function is also monotonic \emph{if and only if}
$w_x \geq 0$ for $x\in\ground$.
Modular functions are also normalized $\setfun(\emptyset)=0$ if $k=0$.

Later, we will prove that composing a monotonic modular function with a real
function
(i.e. a square-root)
produces a set function with the same monotonicity conditions as the
real function.

\subsection{Composing modular functions with real functions}
\label{sec:composing_modular_real}

A key insight of this work is to use composition with real functions such as the
square-root
(as in $\viewreward$ \eqref{eq:view_quality})
to moderate saturation and redundancy of rewards for observations.
In order to obtain guarantees on solution quality,
we must understand when this operation maintains monotonicity properties.

\lemmarealcomposition*

\begin{proof}
  Assume that we can write the derivative of $\hat\setfun$ in the following
  form:
  \begin{align}
    \begin{split}
      &\hat\setfun(Y_1;\ldots;Y_m|X) = \\
      &
      \int_0^{\setfun(Y_1)}
      \cdots
      \int_0^{\setfun(Y_m)}
      \realfun^{(m)}\left(\setfun(X) + \sum\limits_1^m s_m\right)
      \,\mathrm{d}s_m\,\ldots\,\mathrm{d}s_1.
    \end{split}
    \label{eq:composed_derivative}
  \end{align}
  We will later prove this statement by induction.
  This integral is non-negative
  if $\realfun$ is $m$-increasing or else non-positive if $m$-decreasing
  (Def.~\ref{definition:real_monotonic}) and because
  the bounds of the integral satisfy $0 \leq \setfun(Y_i)$ for
  $i\in\{1,\ldots,m\}$ because $\setfun$ is modular and monotonic
  (Def.~\ref{definition:modular}).
  If $\realfun$ is $m$-increasing, then we have
  $\hat\setfun(Y_1;\ldots;Y_m|X) \geq 0$,
  and so $\hat\setfun$ is also $m$-increasing by
  Def.~\ref{definition:set_derivative}.
  Likewise, if $\realfun$ is $m$-decreasing, $\hat\setfun$ is also
  $m$-decreasing.

  All that remains is to prove that we can write the derivative of $\hat\setfun$
  in the form of \eqref{eq:composed_derivative}.
  We can prove this by induction, starting with the base case of $m=1$.
  The derivative is
  \begin{align}
    \hat\setfun(Y_1|X)
    &= \hat\setfun(Y_1,X) - \hat\setfun(X)
    =
    \int_{\setfun(X)}^{\setfun(Y_1,X)} \realfun'(s) \mathrm{d}s
  \intertext{%
    from the definition in Def.~\ref{definition:set_monotonic} and by applying
    the \emph{fundamental theorem of calculus.}
    Then, because $\setfun$ is modular
  }
    \hat\setfun(Y_1|X)
    &=
    \int_{0}^{\setfun(Y_1)} \realfun'(\setfun(X) + s) \mathrm{d}s.
    \label{eq:base_case_derivative}
  \end{align}
  This completes the base case for \eqref{eq:composed_derivative}.
  For the inductive step, we can use \eqref{eq:composed_derivative}
  to write the $m+1$ derivative as
  \begin{align}
    \begin{split}
    &\hat\setfun(Y_1;\ldots;Y_{m+1}|X) \\
    &=
    \hat\setfun(Y_1;\ldots;Y_m|X,Y_{m+1}) - \hat\setfun(Y_1;\ldots;Y_m|X)
    \end{split}
    \\
    \begin{split}
    &=
    \int_0^{\setfun(Y_1)}
    \mkern-20mu\cdots
    \int_0^{\setfun(Y_m)}
    \mkern-10mu
    \realfun^{(m)}\!\!\left(\!
      \setfun(X,Y_{m+1}) + \sum\limits_{i=1}^m s_i
    \!\right)\!\!
    \,\mathrm{d}s_m\,\!\ldots\,\!\mathrm{d}s_1 \\
    &\quad
    -
    \int_0^{\setfun(Y_1)}
    \mkern-20mu\cdots
    \int_0^{\setfun(Y_m)}
    \mkern-10mu
    \realfun^{(m)}\left(\setfun(X) + \sum\limits_{i=1}^m s_i\right)
    \,\mathrm{d}s_m\,\ldots\,\mathrm{d}s_1.
    \end{split}
    \intertext{Combining the terms of the integrands produces}
    \begin{split}
    &=
    \int_0^{\setfun(Y_1)}
    \mkern-20mu\cdots
    \int_0^{\setfun(Y_m)}
    \realfun^{(m)}\left(\setfun(X,Y_{m+1}) + \sum\limits_{i=1}^m s_i\right)
    \\
    & \qquad\qquad
    -
    \realfun^{(m)}\left(\setfun(X) + \sum\limits_{i=1}^m s_i\right)
    \,\mathrm{d}s_m\,\ldots\,\mathrm{d}s_1.
    \end{split}
    \intertext{Now, we can obtain the next integral (with the same approach as
      the base case) to re-write this derivative in the desired form}
    &=
    \int_0^{\setfun(Y_1)}
    \mkern-20mu\cdots
    \int_0^{\setfun(Y_{m+1})}
    \mkern-10mu
    \realfun^{(m+1)}
    \!\left(\!
      \setfun(X) \!+\! \sum\limits_{i=1}^{m+1} s_i
    \!\right)\!
    \mathrm{d}s_{m+1}\ldots\mathrm{d}s_1.
  \end{align}
  And so, we obtain the form for $m+1$
  by assuming \eqref{eq:composed_derivative} for $m$.
  Then, given the base case \eqref{eq:base_case_derivative} of $m=1$
  \eqref{eq:composed_derivative} must hold for all $m\geq1$ by induction.
\end{proof}

\subsection{Proof of Theorem~\ref{theorem:submodular_objective}}
\label{sec:theorem_proof}

Here we prove our main result regarding properties of $\objective$
which ensures bounded suboptimality for our multi-robot planning approach.

\monotonicitiesofsrppa*

\begin{proof}
The proof proceeds in two parts.
First, we prove that SRPPA satisfies alternating monotonicity conditions
and then that this objective is submodular.

\subsubsection{SRPPA satisfies alternating monotonicity conditions}

We can begin by applying Lemma~\ref{lemma:real_composition} so that
$\hat\setfun=\viewreward$ \eqref{eq:view_quality}
and $\setfun=\pixelterm$ \eqref{eq:pixelterm}.
By inspection, we can see that $\viewreward$ was formed by composition with
$\pixelterm$, that $\realfun(x) = p_f A_f \sqrt{x}$, and that $\realfun$ is
defined for $\real_{\geq 0}.$
Additionally, $\pixelterm$ is modular (Def.~\ref{definition:modular}):
the summands in \eqref{eq:pixelterm} are \emph{non-negative} and correspond to
the weights in $\eqref{eq:modular}$.

Therefore, we can apply Lemma~\ref{lemma:real_composition}, and $\viewreward$
has monotonicity properties that match $\realfun$.
The first derivative is $\dot\realfun(x)=p_f A_F x^{-1/2} \geq 0$, and we can
conclude that $\viewreward$ is 1-increasing (monotonic).
Following the \emph{power rule}, we can see that the derivatives of $\realfun$
alternate signs and conclude that $\viewreward$ is submodular
and furthermore $n$-increasing for odd values
of $n$ and $n$-decreasing for even values.

Finally, we must prove that the sum of terms in $\objective$
\eqref{eq:objective} has these same monotonicity properties.
\citet[Sec.~4]{foldes2005} observe that the classes of monotonicity properties
of set functions are closed under conic combination
(linear combinations with non-negative coefficients).
The view rewards $\viewreward_{f,t}$ all satisfy alternating monotonicity
conditions and so does their sum.
Additionally, $\pathreward$ is modular by matching against \eqref{eq:modular}.
As such, $\pathreward$ must be monotonic, submodular, and supermodular
(Def.~\ref{definition:modular}).
Again, following~\citep{foldes2005} modular (or linear) functions have degree 1
\citep[Sec.~4]{foldes2005}
and so further belong to classes of both $n$-increasing and $n$-decreasing
functions for $n\geq 2$ \citep[Sec.~2]{foldes2005}.\footnote{%
  Another path to this conclusion is to observe that for a modular function,
  derivatives (Def.~\ref{definition:set_derivative} of order $n=2$
  and further $n\geq2$ are all identically zero.
}
Therefore, $\pathreward$ satisfies all necessary monotonicity
conditions (and more), and we conclude that the sum of all terms and
so also $\objective$ satisfy alternating monotonicity conditions as stated.

\subsubsection{SRPPA is normalized}
Refer to \eqref{eq:objective}, and recall that $\objective$ is normalized if
$\objective(\emptyset)=0$.
Because $\pathreward$ is modular with zero offset that term is normalized
(Def.~\ref{definition:modular}).
Then, $\viewreward$ \eqref{eq:view_quality} is also clearly normalized because
$\pixelterm$ \eqref{eq:pixelterm} is normalized (modular with zero offset) and
because $\sqrt{0}=0$.
The sum of these terms is therefore also normalized.
\end{proof}
\else\fi

\end{document}